\documentclass[a4paper,	11pt]{article}

\usepackage{bigints}

\usepackage[english]{babel} 
\usepackage[T1]{fontenc} 
\usepackage[babel=true]{csquotes} % csquotes va utiliser la langue d�finie dans babel   
\usepackage{aeguill} 

\usepackage{latexsym}
\usepackage{comment}
\usepackage{natbib}

\usepackage{tikz}
\usepackage{amsmath,amssymb}
\usetikzlibrary{shapes, arrows.meta, positioning, fit}

\def \trans{^{\scriptscriptstyle{\intercal}}}

\usepackage{appendix}
\usepackage[normalem]{ulem} 
\usepackage{upgreek}
\usepackage{tabularx}
\usepackage{amssymb}
\usepackage{amsmath}
\usepackage{amsmath,mathrsfs}
\usepackage{amsthm} 
\usepackage{graphicx}
\usepackage{empheq}
\usepackage{color}
\usepackage{empheq}
\usepackage{float}
\usepackage{caption}
\usepackage{subcaption}
\usepackage{xcolor}
\usepackage{adjustbox}
\usepackage{listings}
\usepackage{algorithm}
\usepackage[algo2e,ruled,vlined]{algorithm2e} 
\usepackage{algorithmic}
\usepackage{bbm}
\usepackage{enumitem}
\usepackage{booktabs}
\usepackage{mathrsfs}
\definecolor{dkgreen}{rgb}{0,0.6,0}
\definecolor{gray}{rgb}{0.5,0.5,0.5}
\definecolor{mauve}{rgb}{0.58,0,0.82}

\def \d{\mathrm{d}} 
\def\eps{\varepsilon} 

\def \E{\mathbb{E}}

\def \P{\mathbb{P}}
\def \Q{\mathbb{Q}}

\def \R{\mathbb{R}}

\def \W{\mathbb{W}}
\newcommand{\bef}[1]{\textbf{#1}}

\usepackage{amsmath, amssymb}
\usepackage{graphicx}
\usepackage{color}
\usepackage{textcomp}
\usepackage{graphics}
\usepackage{float}
\usepackage{empheq}
\usepackage{caption}
\usepackage{subcaption}
\usepackage{hyperref}
\usepackage{mathrsfs}

\usepackage{comment}

\usepackage{geometry}
\usepackage{stmaryrd}

\AtBeginDocument{}
\AtBeginDocument{}

\newtheorem{theorem}{Theorem}[section]

\newtheorem{assum}[theorem]{Assumption}

\mathtoolsset{showonlyrefs}

\newcommand{\be}{\begin{align}}
	\newcommand{\ee}{\end{align}}

\newcommand{\bea}{\begin{eqnarray}}
	\newcommand{\bes}{\begin{subEquations}}
		\newcommand{\ees}{\end{subEquations}}
	\newcommand{\bgt}{\begin{gather}}
		\newcommand{\egt}{\begin{gather}}
			\newcommand{\eea}{\end{eqnarray}}
		\newcommand{\beaa}{\begin{eqnarray*}}
			\newcommand{\eeaa}{\end{eqnarray*}}

		\def \eps{\varepsilon}

		\def \Fc{{\cal F}}

		\def \Lc{{\cal L}}
		\def \Pc{{\cal P}}

		\def \Nc{{\cal N}}

		\def \msY{\mathscr{Y}}

		\def \d{\mathrm{d}} 
		
		\def \trans{^{\scriptscriptstyle{\intercal}}}

		\def\beqs{\begin{eqnarray*}}
			\def\enqs{\end{eqnarray*}}
		\def\beq{\begin{eqnarray}}
			\def\enq{\end{eqnarray}}

		\numberwithin{equation}{section}

\begin{document}
			
			\title{SBBTS: A Unified Schr\"odinger–Bass Framework for Synthetic Financial Time Series}
			
			\author{
				Alexandre ALOUADI \footnote{BNPP and Ecole Polytechnique. This author is supported by a CIFRE industial collaboration between BNP-PAR and Ecole Polytechnique.  \sf alexandre.alouadi@bnpparibas.com}
				\and Grégoire LOEPER \footnote{BNPP  and Monash University \sf gregoire.loeper@bnpparibas.com} 
				\and Célian MARSALA \footnote{BNPP and ENSAE Paris \sf celian.marsala@ensae.fr}
				\and Othmane MAZHAR \footnote{LPSM, Universit\'e Paris Cit\'e and Sorbonne University. This author was supported by the Chair ``Futures of Quantitative Finance".  \sf othmane.xx90@gmail.com}
				\and Huy\^en PHAM  \footnote{Ecole Polytechnique, CMAP. This author is supported by the Chair ``Financial Risks'', by FiME (Laboratory of Finance and Energy Markets), and the EDF–CACIB Chair ``Finance and Sustainable Development''.   \sf huyen.pham@polytechnique.edu}
			}
			
			\date{}
			\maketitle
			
\begin{abstract}
We study the problem of generating synthetic time series that reproduce both marginal distributions and temporal dynamics, a central challenge in financial machine learning. Existing approaches typically fail to jointly model drift and stochastic volatility, as diffusion-based methods fix the volatility while martingale transport models ignore drift. 
We introduce the Schrödinger–Bass Bridge for Time Series (SBBTS), a unified framework that extends the Schrödinger–Bass formulation to multi-step time series. The method constructs a diffusion process that jointly calibrates drift and volatility and admits a tractable decomposition into conditional transport problems, enabling efficient learning.
%We address the problem of generating a continuous semi-martingale with a prescribed 
%joint distribution at successive times $0$ $=$ $t_0$ $<$ $\ldots$ $<$ $t_N$ $=$ $T$. 
%This is formulated as an optimal interpolation problem that unifies both the 
%Schr\"odinger Bridge and Bass frameworks, allowing the construction of a diffusion 
%process that simultaneously calibrates both drift and volatility to time series data. 
%By decomposing the problem into a sequence of optimal transport problems between 
%successive time steps, where mass is transported from a Dirac measure to a given conditional %density, we take advantage of the Schr\"odinger Bridge and Bass (SBB) optimal transport framework %and design an efficient algorithm to solve the SBB problem for time series data. 
Numerical experiments on the Heston model demonstrate that SBBTS accurately recovers stochastic volatility and correlation parameters that prior Schr\"odinger Bridge methods fail to capture. Applied to S\&P 500 data, SBBTS-generated synthetic time series consistently improve downstream forecasting  performance when used for data augmentation, yielding higher classification accuracy and Sharpe ratio compared to real-data-only training. 
These results show that SBBTS provides a practical and effective framework for realistic time series generation and data augmentation in financial applications. 
The code is available at \textcolor{blue}{\href{https://github.com/alexouadi/SBBTS}
{https://github.com/alexouadi/SBBTS}}.
\end{abstract}

			\vspace{5mm}

			\noindent {\bf Keywords}: Machine Learning, Generative AI, Financial Time Series, 
            Schr\"odinger Bridge Bass, Optimal Transport

%\newpage			

\section{Introduction}

The generation of realistic synthetic time series is a central problem in modern machine learning, with applications ranging from finance and healthcare to climate modelling. 
In financial markets, synthetic data are widely used for stress testing, risk management, and training predictive models, especially in settings where data are scarce, costly, or sensitive. However, generating time series that faithfully reproduce both marginal distributions and temporal dynamics remains challenging due to complex dependencies, low signal-to-noise ratios, and the presence of higher-order effects such as stochastic volatility and cross-asset correlations.
%Traditional time series models often demand intensive calibration and are vulnerable to model risk %and estimation errors. Consequently, there is growing interest in developing robust methods for 
%synthetic time series generation. Although generative models have achieved remarkable 
%success in image generation and text generation, time series data—such as financial 
%prices—present unique challenges due to their complex temporal dependencies and low 
%signal-to-noise ratios. In the financial domain, synthetic data are indispensable for 
%simulating market scenarios in stress testing, risk measurement, and deep hedging. As 
%a result, data-driven approaches to synthetic time series generation have attracted 
%considerable attention.

Recent progress in generative modelling, particularly diffusion-based methods, has led to significant advances in high-dimensional data generation. Schrödinger Bridge (SB) methods 
\cite{debortoli2021diffusion} provides a principled framework for constructing stochastic processes that match prescribed marginal distributions by learning a drift that is closest, in a relative entropy sense, to a reference Brownian motion. These approaches have been extended to the interpolation of joint distribution  in \cite{hamdouche2023generative} and have shown promising results for time series generation. However, a key limitation of SB methods is that the volatility structure is fixed by construction, which prevents them from capturing important features of financial data such as stochastic volatility and correlated noise.

An alternative perspective is provided by martingale transport methods, in particular the Bass framework, which focuses on calibrating the volatility to match marginal distributions while constraining the drift, see   \cite{BackhoffVeraguas2019, ConzeHenryLabordere2019},  
\cite{acciaio2023localvolatility, joseph2024measurepreserving}.  
While effective for certain calibration problems, this approach ignores drift dynamics and therefore fails to capture temporal dependencies and predictive structure. As a result, neither framework alone is sufficient to model realistic time series where both drift and volatility play a fundamental role.
%Two prominent frameworks address the calibration of a diffusion process to prescribed 
%marginal distributions. The Schrödinger bridge (SB) framework 
%\cite{schrodinger1931, Follmer1988} finds the drift of a diffusion minimizing the 
%Kullback-Leibler divergence from a reference Brownian motion under marginal 
%constraints, and has been leveraged for synthetic time series generation 
%\cite{debortoli2021diffusion, hamdouche2023generative}. The Bass local volatility 
%model \cite{BackhoffVeraguas2019, ConzeHenryLabordere2019} instead calibrates the 
%volatility to match a finite set of marginals via a Brownian martingale, converging 
%to the classical Dupire model \cite{Dupire1994} as the time grid is refined 
%\cite{acciaio2023localvolatility, joseph2024measurepreserving}. 
%However, the SB approach fixes the volatility by construction, while the Bass approach constrains 
%the drift to zero: neither framework jointly calibrates both components — a critical 
%limitation for financial time series where drift and stochastic volatility carry 
%equally essential information. 

The Schrödinger–Bridge-Bass (SBB) framework was recently introduced in \cite{henrylabordere2026bridgingschrodingerbasssemimartingale} to bridge this gap by jointly optimizing over drift and volatility through a unified optimal transport formulation. By interpolating between the SB and Bass regimes via a tunable parameter, SBB provides a flexible mechanism to capture both components of the dynamics. However, existing results are restricted to the two-marginal setting and do not directly extend to full time series distributions.
%The Schrödinger–Bass (SBB) problem  \cite{henrylabordere2026bridgingschrodingerbasssemimartingale} %unifies both paradigms by jointly optimizing over drift  and volatility via a tunable parameter %$\beta > 0$ that interpolates between the  two regimes, and %\cite{alouadi2026lightsbbmbridgingschrodingerbass} provides an efficient algorithm to solve 
%it in the two-marginal case. The present work extends this framework to the full 
%time series setting.

%\paragraph{Contributions.} 
In this paper, we introduce a new framework for synthetic time series generation that combines optimal transport with modern machine learning techniques. Our approach is designed to reproduce both marginal distributions and temporal dynamics—two key ingredients for realistic time series modelling—by constructing a continuous-time process that interpolates the joint distribution across successive time steps. We extend the Schrödinger-Bass Bridge problem from the two-marginal setting to full time series distributions, enabling the joint calibration of drift and volatility. We further show that the resulting problem, called the Schrödinger–Bass Bridge for Time Series (SBBTS), 
admits a decomposition into a sequence of conditional optimal transport problems, making it computationally tractable. Building on this structure, we design a scalable neural implementation that captures path-dependent dynamics. Finally, we demonstrate empirically that the proposed method accurately recovers stochastic volatility and correlation structures and improves downstream forecasting performance when used for data augmentation on real financial data. 

The remainder of the paper is organised as follows. In Section 2, we review the Schrödinger Bridge and Bass frameworks and introduce the Schrödinger–Bass (SBB) problem. Section 3 formulates the SBBTS problem for time series and presents a key decomposition result that reduces it to a sequence of conditional transport problems. Section 4 describes the proposed neural algorithm and training procedure. Section 5 provides empirical evaluations on both synthetic benchmarks and real financial data, including data augmentation experiments. Finally, Section 6 concludes and discusses limitations and future research directions.

\paragraph{Notations.} 
\begin{itemize}
\item  A  random variable $X$ distributed according to a probability measure $\nu$  is denoted $X$ $\sim$ $\nu$, and $\E_\nu$ is the expectation operator under $\nu$, i.e., $\E_\nu[\varphi(X)]$ $=$ $\int \varphi \d\nu$.  
%We denote by $L^1(\nu)$ the set of integrable functions w.r.t. the measure $\nu$ on $\R^d$, and 
For a measurable function $\varphi$ on $\R^d$, $\varphi\#\nu$  is the pushforward measure of $\nu$.  
When $X,Y$ are random variables on a probability space $(\Omega,\Fc,\P)$, we also denote $\P\circ X^{-1}$ $=$ $X\#\P$ the law of $X$ under $\P$. 
We denote by $\mu*\nu$ the convolution of two probability measures $\mu$, $\nu$, i.e., the law of $X+Y$ when $X$ $\sim$ $\mu$, $Y$ $\sim$ $\nu$ are independent. 
For a measurable function $\phi$ on $\R^d$, and a probability measure $\mu$ on $\R^d$, we denote by $\mu*\phi$ the function defined on $\R^d$ by $\mu*\phi(x)$ $=$ $\int \phi(x+y)\mu(\d y)$. 
\item  $\Nc_t$ is the normal distribution of mean $0$ and covariance matrix $t I_d$, $t$ $>$ $0$, where $I_d$ is the identity matrix in $\R^{d\times d}$.
\end{itemize}

\section{Background: Schr\"odinger Bridge Bass Problem} \label{sec:SBB_recall}

The Schr\"odinger-Bridge-Bass (SBB) problem, introduced and studied in \cite{henrylabordere2026bridgingschrodingerbasssemimartingale}, is an extension of the classical Schr\"odinger Bridge (SB) problem by jointly optimizing over both drift and volatility of diffusion process. Denote by $\Pc$ the set of probability measure $\P$ on the canonical space $\Omega$ $=$ $C([0,T],\R^d)$ under which the canonical process $X$ has the diffusion decomposition 
\begin{align} \label{decdiffusion} 
X_t &= \;  X_0 + \int_0^t \alpha_s \d s + \int_0^t \sigma_s \d W_s, \qquad  t \in [0,T], \;\;  \P-\mbox{a.s.} 
\end{align} 
with $W$ a $d$-dimensional Brownian motion under $\P$. Now, 
given two probability distributions $\mu_0, \mu_T$ on $\R^d$ with second-order moments,  the goal is to find  $\P$ $\in$ $\Pc$, which minimizes the quadratic cost
\begin{align} \label{defJ} 
J(\P) &= \;  \E_\P \Big[ \int_0^T \| \alpha_t \|^2 + \beta \| \sigma_t - I_d \|^2\, \d t \Big],  
\end{align}
under the marginal constraints $\P\circ X_0^{-1}$ $=$ $\mu_0$ and $\P\circ X_T^{-1}$ $=$ $\mu_T$. We denote by $\Pc(\mu_0,\mu_T)$ the set of such probability measures $\P$ on $\Omega$, and the optimal value of such problem by  
\beq \label{defSBB}
{\rm SBB}(\mu_0,\mu_T) &:=& \inf_{\P \in \Pc(\mu_0,\mu_T)} J(\P).
\enq            
Formally, when $\beta$ goes to infinity,  we constrain the volatility coefficient $\sigma$ to be equal to $I_d$, and we then search for the drifted Brownian motion  $\d X_t$ $=$ $\alpha_t \d t + \d W_t$, 
that is closest to the Brownian motion with respect to the relative entropy (Kullback-Leibler) distance,  under the marginal distribution constraints $X_0$ $\sim$ $\mu_0$, $X_T$ $\sim$ $\mu_T$.  This is the classical Schr\"odinger bridge problem. 
At the other extreme, by dividing the criterion $J$ by $\beta$, and sending $\beta$  to zero, we formally constraint the drift coefficient to be zero, and then we are looking for a Brownian martingale which is closest to the Brownian motion according to the quadratic norm, under the  marginal distribution constraints. This is the Bass martingale transport problem  studied in  \cite{ConzeHenryLabordere2019, acciaio2023localvolatility,  bacetal23b}, motivated by calibration problems. In other words, the parameter $\beta$ controls the relative weight of drift versus volatility, interpolating between these two regimes. 

\vspace{1mm}

The solution of the SBB problem is expressed in terms of a triple $(h,\nu,\msY)$ 
of density/measure/transport map  satisfying a backward/forward/transport structure: 
\begin{equation}
\begin{cases}
h_t &=\; h_T * \Nc_{T-t} \\
\nu_t &= \; \nu_0 * \Nc_t \\
\msY_t & = (\nabla_y \Phi_t)^{-1}, 
\quad \Phi_t(y) = \frac{|y|^2}{2} + \frac{1}{\beta} \log h_t(y)
\end{cases}
\end{equation}
and the endpoints conditions, called  SBB system: 
\beqs
\begin{cases}
\msY_T \# \mu_T &= \;  h_T  \nu_T,   \\
\msY_0 \# \mu_0 & = \;  h_0  \nu_0. 
\end{cases}
\enqs
Existence of such a triple $(h,\nu,\msY)$ satisfying the SBB system is shown in \cite{henrylabordere2026bridgingschrodingerbasssemimartingale} under the condition that 
$\beta T$ $>$ $1$. In this case, and under the finite relative entropy assumption 
\beqs
{\rm KL}(\mu_T|\mu_0*\Nc_T) &:= \; \E_{\mu_T}\big[ \log \frac{\d\mu_T}{\d (\mu_0*\Nc_T)}] \; < \; \infty, 
\enqs
there exists a solution $\P^{SBB}$ to the SBB problem \eqref{defSBB}, with 
an  optimal drift and vola\-tility  given by 
			\beqs
			\begin{cases}
				\alpha_t^* &= \;    \nabla_y \log h_t(\msY_t(X_t)), \\  
				\sigma_t^* &= \;  D_y^2 \Phi_t(\msY_t(X_t)), 
			\end{cases}
			\quad t \in [0,T]. 
			\enqs 
%and we denote by $\P^{SBB}$ $\in$ $\Pc(\mu_0,\mu_T)$ the law %of the optimal diffusion process $X$ on the path space %$\Omega$. 			
Moreover, if we define the process 
			\beq  
			\label{eq:msY_def}
			Y_t = \msY_t(X_t) = X_t - \frac{1}{\beta}  \nabla_y \log h_t(\msY_t(X_t)), \quad  t \in [0,T],
			\enq
			and the change of measure $\frac{\d\hat\Q}{\d\P^{\rm SBB}} \Big|_{\Fc_t} = \frac{1}{h_t(Y_t)}, \; t \in [0,T]$,  
            %(which is indeed a $\P^{\rm SBB}$-martingale with %expectation $1$ by the SBB system), 
            then
			\begin{itemize}
				\item $(Y_t)_t$ is a Brownian motion under $\hat\Q$ with initial law $\nu_0$, and is a diffusion Sch\"odinger bridge (DSB)  under $\P^{SBB}$: 
\beq \label{DSBY} 
\d Y_t &=& \nabla_y \log h_t(Y_t) \d t + \d W_t, \quad Y_0 \sim \msY_0\#\mu_0, \; Y_T \sim \msY_T\#\mu_T. 
\enq
\item $X_t = \msY_t^{-1}(Y_t) = Y_t + \frac{1}{\beta} \nabla_y \log h_t(Y_t)$, $t \in [0,T]$, is a stretched Brownian motion under $\hat\Q$, and a stretched diffusion Schr\"odinger bridge under $\P^{SBB}$. 
%\item The dynamics under $\P^{\rm SBB}$ are:
%\begin{equation*}
%\begin{cases}
%\d X_t = D_y^2 \Phi_t(Y_t) \, \d Y_t, \\
%\d Y_t = \varepsilon \nabla_y \log h_t^*(Y_t) \, \d t + \sqrt{\varepsilon} \, \d W_t, 
%\end{cases} 
%\quad t \in [0,T].
%\end{equation*}
\end{itemize}
			
To generate new samples from $\mu_T$ through the learned SBB system, the process $Y = \msY(X)$ can be generated as a DSB from $\msY_0\#\mu_0$ to $\msY_T\#\mu_T$, 
%\begin{equation}
%\label{eq:Y_process}
%\d Y_t = \varepsilon \nabla_y \log h_t^*(Y_t) \, \d t + \sqrt{\varepsilon} \, \d W_t,
%\end{equation}
with score drift $s_t(Y_t) = \nabla_y \log h_t(Y_t)$. Then, $X_T$ can be recovered by
\beqs
X_T &=& \msY_T^{-1}(Y_T) = Y_T + \frac{1}{\beta} \nabla_y \log h_T(Y_T) \sim \mu_T.
\enqs

\section{Schr\"odinger Bridge Bass  for Time Series Problem}

In this section, we extend the SBB problem to time series framework. We are now given 
a joint distribution $\mu$ $\in$ $\Pc((\R^d)^{n+1})$ corresponding to the law of a time series on $\R^d$ observed at $n+1$ dates $t_0$ $=$ $0$ $<$ $\ldots$ $<$ $t_i$ $<$ $\ldots$ $t_n$ $=$ $T$.

\subsection{Problem Formulation} \label{sec: Schr\"odinger Bass Bridge time series problem}

We aim to construct on the canonical space $\Omega$ $=$ $C([0,T],\R^d)$ a probability measure $\P$ $\in$ $\Pc$, which minimizes the quadratic cost $J(\P)$ as in \eqref{defJ}, but now under the constraint that 
$\P$ $\circ$ $(X_{t_0},\ldots,X_{t_n})^{-1}$ $=$ $\mu$, i.e.,  $(X_{t_0},\ldots,X_{t_n})$ $\sim$ $\mu$ under $\P$. We denote by $\Pc(\mu)$ the set of such probability measures $\P$ satisfying this  joint distribution  constraint, and its optimal value by   
\begin{align} \label{defSBBTS} 
{\rm SBBTS}(\mu) &  := \; \inf_{\P\in\Pc(\mu)} J(\P). 
\end{align}
Problem \eqref{defSBBTS} is called  the Schr\"odinger  bridge Bass time series interpolation problem, and  the solution to ${\rm SBBTS}(\mu)$: $X_t$ $=$ $X_0 + \int_0^t \alpha_s^* \d s + \int_0^t \sigma_s^* \d W^*_s$, $0\leq t\leq T$, is called SBBTS diffusion process. 
				
To handle the joint distribution constraint, we exploit its factorization into conditional distributions across time. 	In the sequel, for $(X_{t_0},\ldots,X_{t_n})$ $\sim$ $\mu$, we set $X_{t_0:t_i}$ $:=$ $(X_{t_0},\ldots,X_{t_i})$ for $i$ $=$ $0,\ldots,n$, and denote by  
				$\mu_i$ the law of $X_{t_0:t_i}$,  and by  $\mu_{i+1|0:i}(\cdot|x_{0:i})$ the conditional distribution of $X_{t_{i+1}}$ given $X_{t_0:t_i}$ $=$ $x_{0:i}$ $:=$ $(x_0,\ldots,x_i)$ $\in$ $(\R^d)^{i+1}$. We then have the chain  probability rule: 
				$\mu_{i+1}$ $=$ $\mu_i\mu_{i+1|0:i}$, and so  $\mu$ $=$ $\mu_n$ $=$ $\mu_0\prod_{i=0}^{n-1} \mu_{i+1|0:i}$. We also note the time step $\Delta t_i = t_{i+1} - t_i$.

				\subsection{Explicit Construction of the Solution to SBBTS}

				\label{sec: Explicit construction of the solution to SBBTS} 
				
				% \subsection{Assumptions}

				We make the following assumptions.

				\begin{assum} \label{HypSBBTS}
					For any $x_{0:i}$ $:=$ $(x_0,\ldots,x_i)$ $\in$ $(\R^d)^{i+1}$, we make the standing assumption  that $\mu_{i+1|0:i}$ $=$ $\mu_{i+1|0:i}(\cdot|x_{0:i})$ has finite second moment,  and that 
					%\beqs
					%\E_{\mu_{i+1|0:i}}  \| X_{t_{i+1}}\|^2  %&=&  \int \|x\|^2 \mu_{i+1|0:i}(\d x| %x_{0:i}) \; < \; \infty. 
					%\enqs
					%Moreover,  we assume 
                    it is  absolutely continuous w.r.t. $\Nc_{t_{i+1}-t_i}$ with a positive and continuous Radon-Nikodym density 
					$\frac{\d\mu_{i+1|0:i}}{\d\Nc_{t_{i+1}-t_i}}$ on $\R^d$, and finite relative entropy (or Kullback-Leibler distance): 
					\beqs
					{\rm KL}(\mu_{i+1|0:i}| \Nc_{t_{i+1}-t_i}) &:= & \int  \big[ \log \frac{\d\mu_{i+1|0:i}}{\d\Nc_{t_{i+1}-t_i}} \big] \d \mu_{i+1|0:i}  \; < \; \infty. 
					\enqs
				\end{assum}

				We show that the optimal interpolation problem of a joint distribution can be reduced to a sequence of classical semimartingale optimal transport problems on each interval $[t_i,t_{i+1})$, $i$ $=$ $0,\ldots,n-1$ with marginal constraints.  More precisely, we have the following decomposition result:

\begin{theorem} \label{propreduc} 					
Assume that $\beta \Delta t_i$ $>$ $1$ for all $i$ $=$ $0,\ldots,n-1$. We have 
\begin{align}
{\rm SBBTS}(\mu)  &= \; \E_\mu \Big[ \sum_{i=0}^{n-1} V_i(X_{t_0:t_i}) \Big] =    \int   \sum_{i=0}^{n-1}  V_i(x_{0:i}) \mu(\d x_{0:n}) \; = \;  \sum_{i=0}^{n-1} \int  V_i(x_{0:i}) \mu_i(\d x_{0:i}), 
\end{align} 
where 
\begin{align} \label{defVi} 
V_i(x_{0:i}) & =  {\rm SBB}(\delta_{x_i},\mu_{i+1|0:i}(\cdot|x_{0:i})) 
             =  \inf_{\scriptstyle \P \in \Pc^i(\mu_{i+1|0:i}(\cdot|x_{0:i}))}
						% \inf_{(\alpha,\sigma) \in \Ac_{\mu_i}} 
						% \E\Big[  \int_{t_i}^{t_{i+1}} H(\alpha_t,\sigma_t)  \d t \Big], 
						% \E\Big[ \frac{1}{2} \int_{t_i}^{t_{i+1}}  |\alpha_t|^2_{}  + \beta |\sigma_t - I_d |^2 \d t \Big], 
						\E_\P\Big[ \frac{1}{2} \int_{t_i}^{t_{i+1}}  \|\alpha_t\|^2  + \beta \|\sigma_t - I_d \|^2 \d t  
						%\big| X_{t_0:t_i} = x_{0:i} 
						\Big], 
\end{align}
and 
					% $\Ac^i_{\mu_{i+1|0:i}(\cdot|x_{0:i})}$ 
$\Pc^i(\mu_{i+1|0:i}(\cdot|x_{0:i}))$  is the set of elements $\P$ $\in$ $\Pc$ s.t.   $\P\circ X_{t_i}^{-1}$ $=$ $\delta_{x_i}$ and $\P\circ X_{t_{i+1}}^{-1}$ $=$ $\mu_{i+1|0:i}(\cdot|x_{0:i})$.  
\end{theorem}
				
				The proof of Theorem~\ref{propreduc} can be found in Appendix~\ref{sec:proof_theorem}. The dynamic programming type decomposition  in the above proposition shows that  the diffusion solution to the SBBTS problem can be constructed sequentially from the resolution of  the  optimal transport problems $V_i$ and concatenating the processes defined on the intervals $[t_i, t_{i+1}]$ for $i = 0 ,\ldots,n-1$. 
				Specifically, at  time step  $t_i$, after computing the optimal values $(\alpha_t^*, \sigma_t^*)_{t=0}^{t_i}$, we can simulate the process over the time interval $[0, t_i]$, and encode the obtained values  $X_{t_0:t_i} = x_{0:i}$. Then, we solve the optimal transport problem 
				$V_i(x_{0:i})$ that transports the Dirac measure $\delta_{x_i}$ at time $t_i$ to the measure $\mu_{i+1|0:i}$ at time $t_{i+1}$, to get   $(\alpha_t^*, \sigma_t^*)_{t = t_i}^{t_{i+1}}$, and continue this process until a solution over the entire interval $[0, T]$ is obtained. 
				%In the next section, we present the  resolution of the static optimal transport problem $V_i$.  

				\section{Algorithm for the SBBTS Problem}
				
				As mentioned in Section~\ref{sec:SBB_recall},  one may generate the auxiliary process $Y$  in \eqref{eq:msY_def}, which solves a classical SB, and then recover $X$ via the inverse transport map. 
				In practice, the parameter $\beta$ is never chosen too small. Indeed, the constraint
				$\beta > \frac{1}{\Delta t_i}$ together with the typical time resolution of financial time
				series makes large values of $\Delta t_i$ undesirable. In this regime, following
				\cite{alouadi2026lightsbbmbridgingschrodingerbass}, the transport map admits the large-$\beta$ approximation:
				\begin{align}
					\label{eq:msY_approx}
					\msY_t(x)
					= x -  \frac{1}{\beta} \nabla_y \log h_t(\msY_t(x))
					\simeq x - \frac{1}{\beta} \nabla_y \log h_t(x),
					\qquad t \in [t_i,t_{i+1}].
				\end{align}
				
				We therefore follow the general structure of the large-$\beta$ algorithm proposed in
				\cite{alouadi2026lightsbbmbridgingschrodingerbass}. However, we found the Light-SB approach to be insufficiently flexible
				for time series data, as the weights of the Gaussian mixture are fixed. Instead, we
				parametrize the drift using a neural network $s_{\theta}$, which takes as inputs the
				current time $t \in \R$, the current state $Y_t \in \R^d$, and an embedding vector encoding
				the past trajectory. More precisely, for each $i$, we define
				\[
				c_i := \Phi_{\theta}(Y_{t_0:t_i}),
				\]
				where $\Phi_{\theta}$ is an encoder-only network. We illustrate in Figure~\ref{fig:netw} the architecture of the neural network $s_\theta$ used to parametrize the drift, and more details can be found in Appendix~\ref{sec:nn_archi}. \\
				
				\begin{figure}
					\centering
                    \includegraphics[width=\columnwidth]{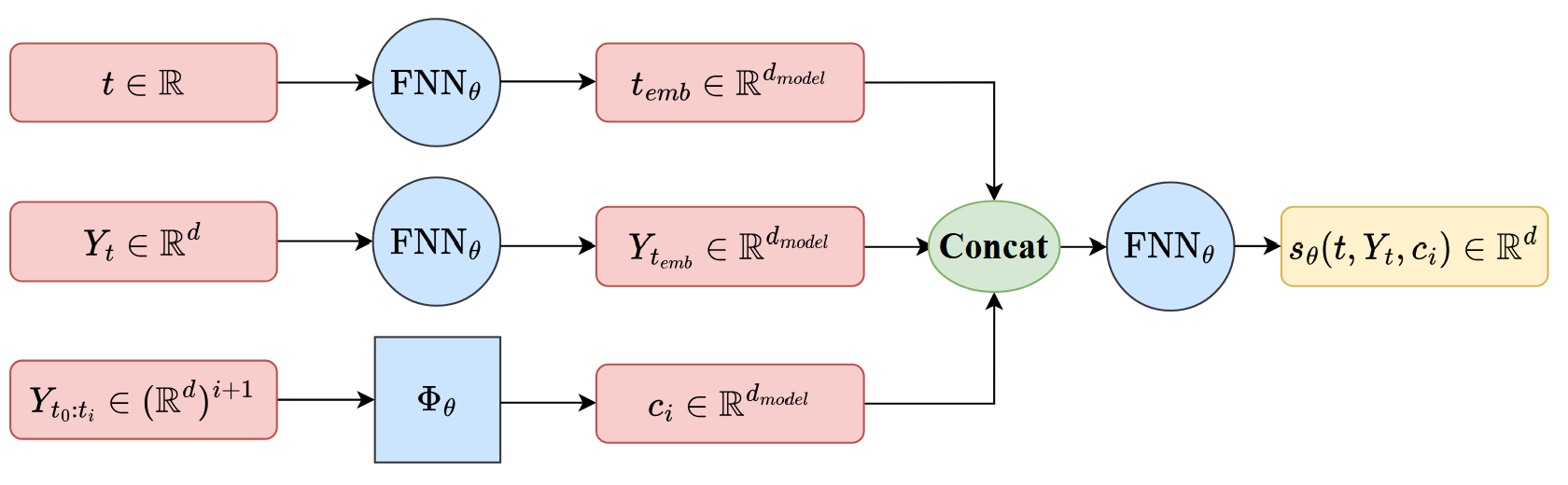}
					\caption{Architecture of the model $s_{\theta}$}
					\label{fig:netw}
				\end{figure}

				The parameters $\theta$ are learned by minimizing the following loss function, averaged
				over all time intervals:
				\begin{align}
					\label{eq:loss_training}
					\Lc(\theta)
					=
					\frac{1}{N}
					\sum_{i=0}^{N-1}
					\E_{t \sim \mathcal{U}([t_i, t_{i+1}))}
					\E_{\substack{
							Y_{t_{i+1}} \sim \mu_{i+1|0:i}, \\
							Y_t \sim \W_{|Y_{t_i}, Y_{t_{i+1}}}
					}}
					\left[
					\left\|
					s_{\theta}\bigl(t, Y_t, \Phi_{\theta}(Y_{t_0:t_i})\bigr)
					-
					\frac{Y_{t_{i+1}} - Y_t}{t_{i+1} - t}
					\right\|_2^2
					\right]
				\end{align}
				
				Here, $\W_{|y_{t_i}, y_{t_{i+1}}}$ denotes the law of the Brownian bridge between $y_{t_i}$ and $y_{t_{i+1}}$. Explicitly, for
				$t \sim \mathcal{U}([t_i, t_{i+1}))$ and $Z \sim \mathcal{N}(0,I_d)$,
				\begin{align}
					\label{eq:brownian_bridge}
					y_t
					=
					\frac{t_{i+1} - t}{\Delta t_i} y_{t_i}
					+
					\frac{t - t_i}{\Delta t_i} y_{t_{i+1}}
					+
					\sigma_t \, Z,
					\qquad
					\sigma_t^2
					=
					\frac{(t - t_i)(t_{i+1} - t)}{\Delta t_i}.
				\end{align}
				
				As in \cite{alouadi2026lightsbbmbridgingschrodingerbass}, the transport map is updated iteratively and initialized as the
				identity. This choice is natural in the present setting, since we consider moderately
				large values of $\beta$, corresponding to a regime close to the classical SB, for which $\msY = I_d$. The complete training procedure is summarized in Algorithm~\ref{alg:sbb_algo_beta_large}.
				
				\begin{algorithm}[H]
					\caption{SBBTS training algorithm}
					\label{alg:sbb_algo_beta_large}
					\begin{algorithmic}[1]
						\REQUIRE Samples $(X_{t_0}^m, \cdots, X_{t_N}^m)_{m \leq M} \sim \mu$, $\theta$,  $\beta$,  $K>0$, batch
						size $B$
						\STATE \textbf{Initialization:} set $\msY^0 = I_d$ (hence $s_\theta^0 \equiv 0$)
						\FOR{$k = 0, \cdots, K-1$}
						\REPEAT
						\STATE Draw a mini-batch $(X_{t_0}^b, \cdots, X_{t_N}^b)_{b \leq B}$
						\vspace{0.2em}
						\STATE Compute
						\[
						\left\{
						Y_{t_i}^b
						=
						X_{t_i}^b
						-
						\frac{1}{\beta}
						s_\theta^k(t_i, X_{t_i}^b, \Phi_\theta^k(X_{t_0:t_i}^b))
						\sim
						\msY_{t_i}^k \# \delta_{X_{t_i}^b}
						\right\}_{i \leq N-1}
						\]
						\STATE Compute
						\[
						\left\{
						Y_{t_{i+1}}^b
						=
						X_{t_{i+1}}^b
						-
						\frac{1}{\beta}
						s_\theta^k\bigl(t_{i+1}, X_{t_{i+1}}^b, \Phi_\theta^k(X_{t_0:t_i}^b)\bigr)
						\sim
						\msY_{t_{i+1}}^k \# \mu_{i+1|0:i}
						\right\}_{i \leq N-1}
						\]
						\STATE Sample
						$
						\{Y_t^b \sim \mathbb{W}^{}_{|Y_{t_i}^b, Y_{t_{i+1}}^b}\}_{i \leq N-1}
						$
						using~\eqref{eq:brownian_bridge}
						\STATE Update $\theta^k$ by minimizing~\eqref{eq:loss_training}
						\UNTIL{convergence}
						\STATE $\theta^{k+1} \leftarrow \theta^k$
						\ENDFOR
						\STATE \textbf{Return} $\theta^K$
					\end{algorithmic}
				\end{algorithm}
				
				Once the drift $s_\theta^K \simeq \nabla_y \log h$ has been learned, new time series
				samples can be generated as follows. First compute
				\[
				Y_{t_0} = X_{t_0} - \frac{1}{\beta} s_\theta^K(t_0, X_{t_0}, \Phi_\theta^K(X_{t_0})).
				\]
				Then simulate the dynamics~\eqref{DSBY} on the interval $[t_0,t_1)$ using the
				drift $s_\theta^K\bigl(t, X_t, \Phi_\theta^K(Y_{t_0})\bigr)$, and recover
				\[
				X_{t_1}
				=
				Y_{t_1}
				+
				\frac{1}{\beta}
				s_\theta^K\bigl(t_1, Y_{t_1}, \Phi_\theta^K(Y_{t_0})\bigr).
				\]
				Starting from $Y_{t_1}$, the procedure is repeated sequentially to obtain $Y_{t_2}$ and
				so on.
				
				Note that the target score is not well defined at $t = t_{i+1}$. In practice, relying on
				the continuity of $\log h$, we evaluate it instead at
				$\tilde t_{i+1} = t_{i+1} - \xi$, for some $\xi > 0$.

				\section{Numerical Experiments}
				
				In this section, we empirically assess the effectiveness of the SBBTS algorithm on a variety of time series models, ranging from low-dimensional synthetic examples to high-dimensional real-world datasets, with applications to time series forecasting. The general implementation settings are described in Appendix~\ref{sec:details_num_exp}.

				\subsection{Heston Process}
				
				In this part, we follow the experimental framework introduced in
				\cite{Alouadi_2025} to assess the robustness of the SBBTS model. The objective is to
				recover the parameters of the parametric two-dimensional Heston model with stochastic volatility, defined by
				\[
				\begin{cases}
					dX_t = r X_t \, dt + \sqrt{v_t}\, X_t \, dW_t^X, \\
					dv_t = \kappa(\theta - v_t)\, dt + \xi \sqrt{v_t}\, dW_t^v,
				\end{cases}
				\]
				where $\kappa > 0$, $\theta > 0$, $\xi > 0$, $r \in \R$, and
				$\rho := \mathrm{Corr}(W_t^X, W_t^v) \in [-1,1]$ denote the model parameters.
				
				In this setting, each parameter vector is independently sampled from a prescribed
				range, so that the training dataset consists of Heston time series generated under
				heterogeneous parameter configurations. The generative model is then fit on this
				dataset to generate new synthetic time series. Finally, the Heston parameters are
				estimated on each generated sample using a maximum likelihood approach, allowing
				us to evaluate the ability of the model to preserve the underlying parametric
				structure. In our experiments, we use $5000$ real trajectories of length $252$ for training and generate a synthetic dataset of $5000$ trajectories. Moreover, we benchmarked the results of SBBTS with the SBTS model \cite{hamdouche2023generative}. \\

				\begin{figure}[htbp]
					\centering
					
					% First row: large images
					\begin{subfigure}{0.48\textwidth}
						\centering
						\includegraphics[width=\linewidth]{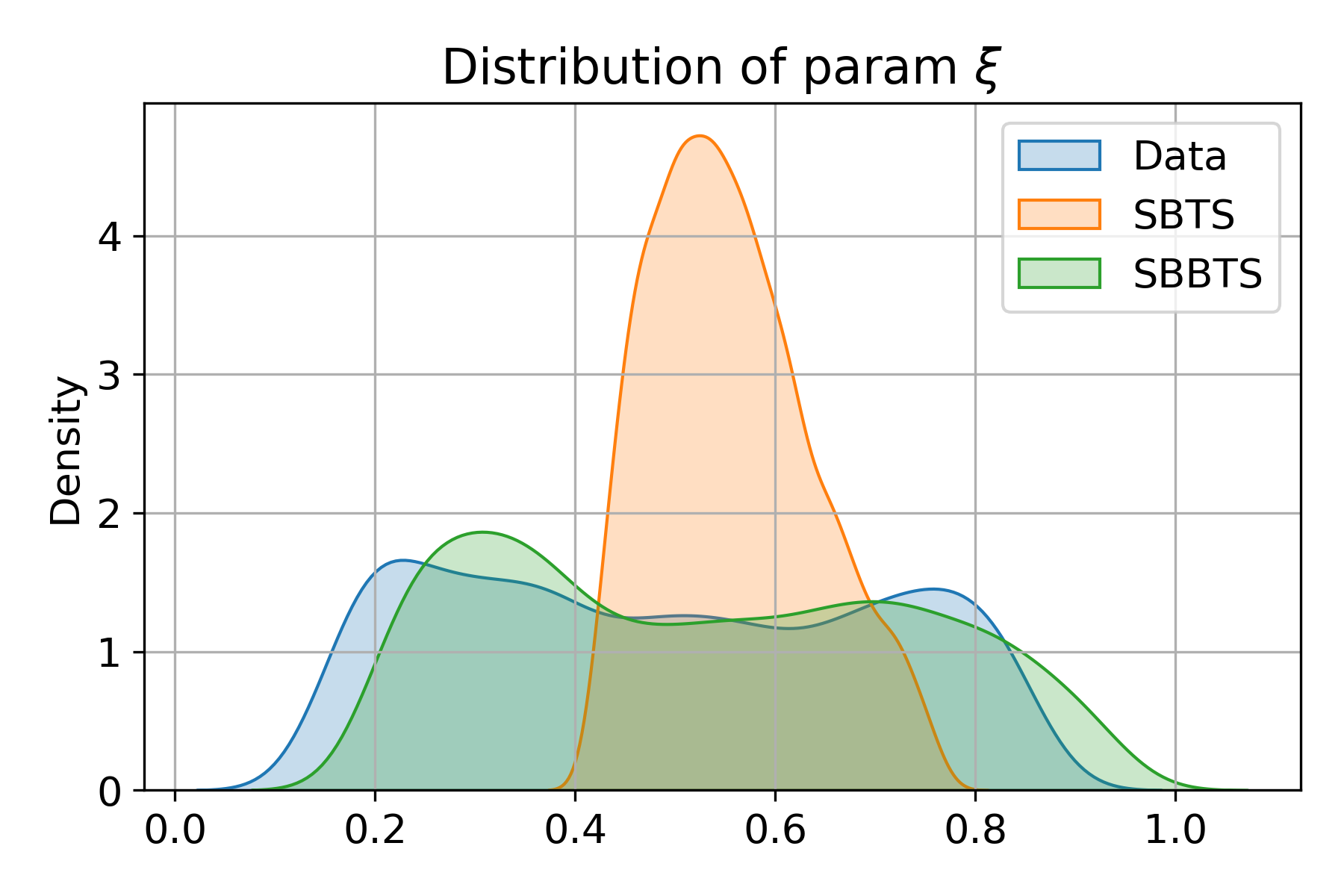}
					\end{subfigure}
					\hfill
					\begin{subfigure}{0.48\textwidth}
						\centering
						\includegraphics[width=\linewidth]{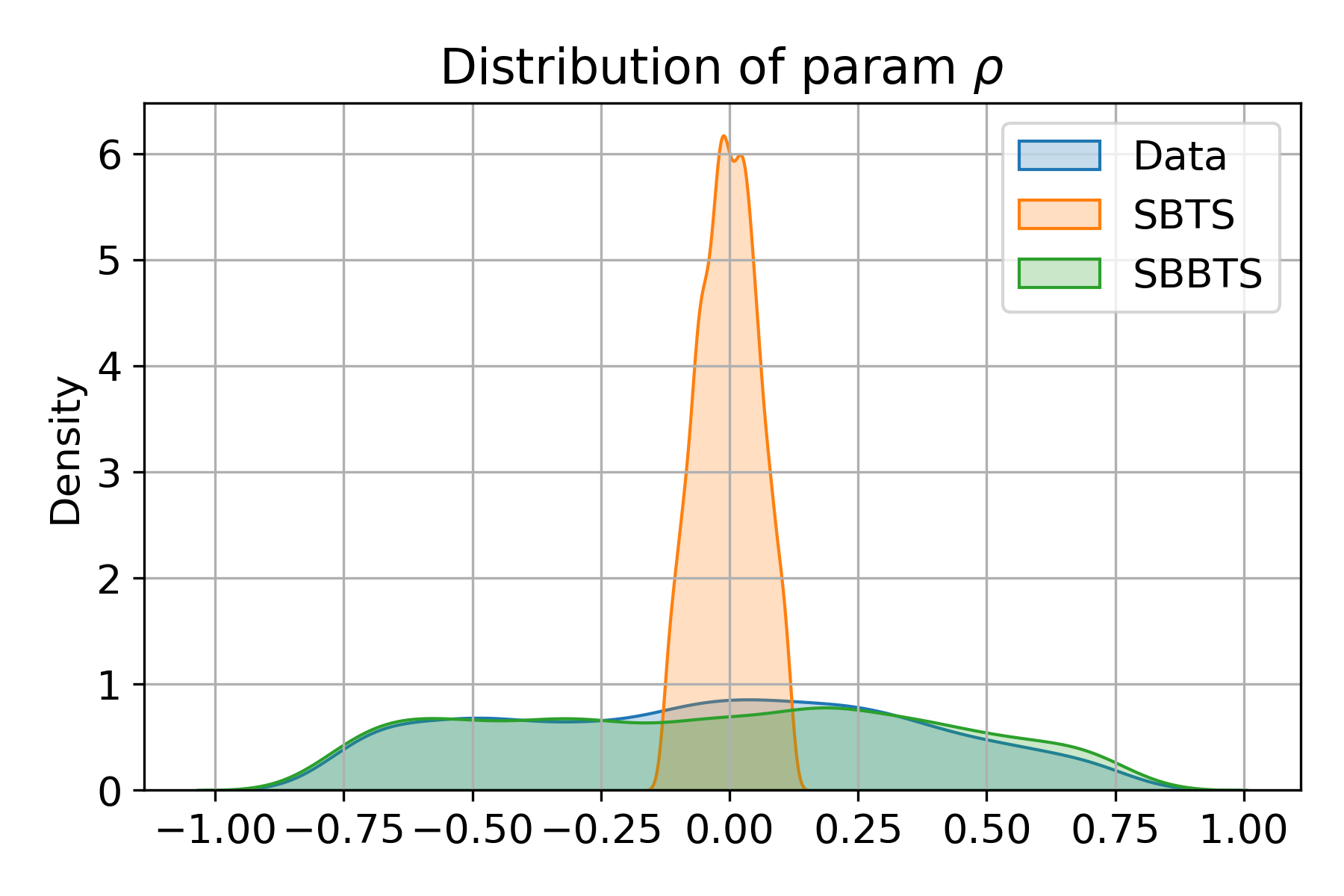}
					\end{subfigure}
					
					\vspace{0.5em}
					
					% Second row: smaller images
					\begin{subfigure}{0.3\textwidth}
						\centering
						\includegraphics[width=\linewidth]{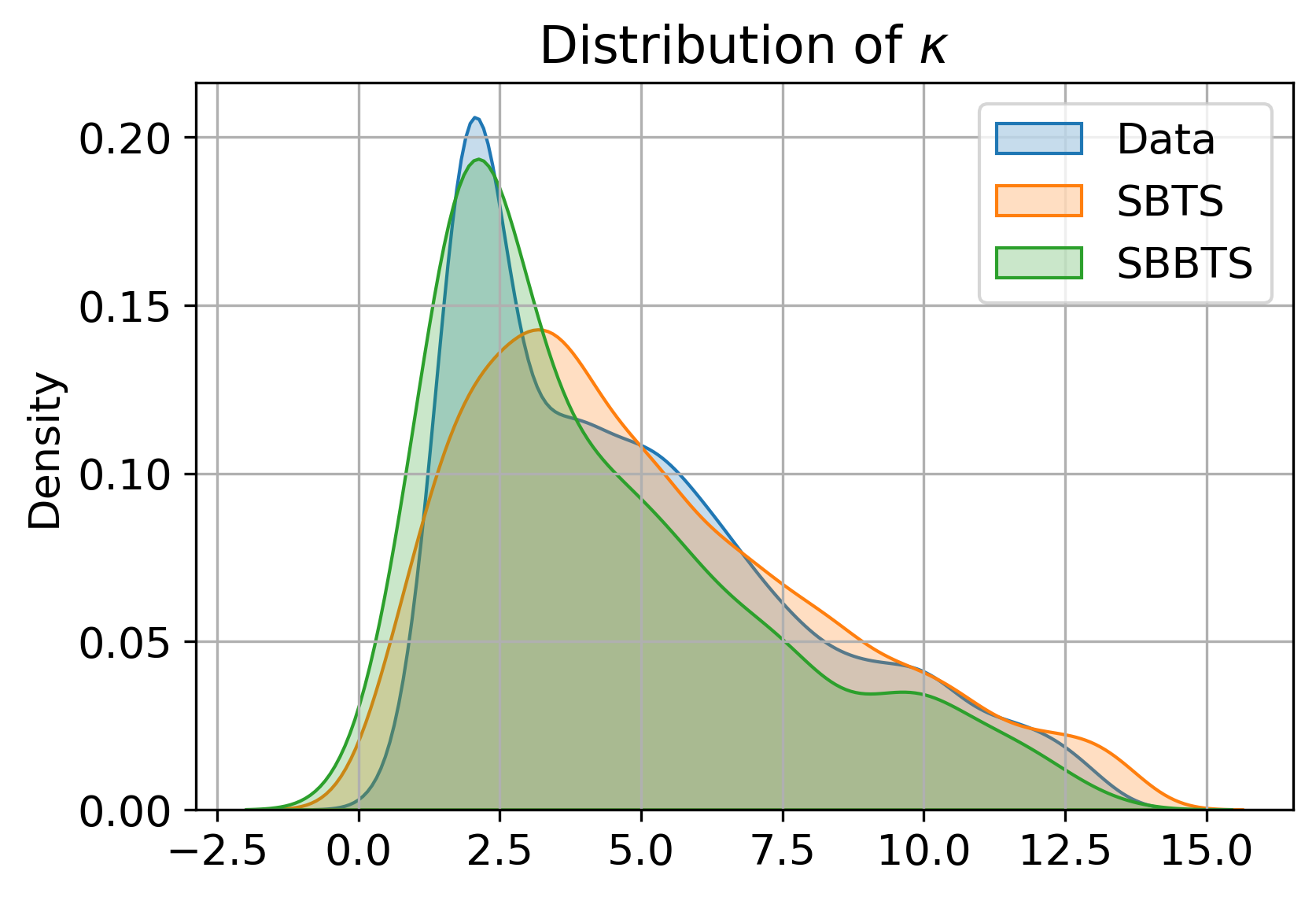}
					\end{subfigure}
					\hfill
					\begin{subfigure}{0.3\textwidth}
						\centering
						\includegraphics[width=\linewidth]{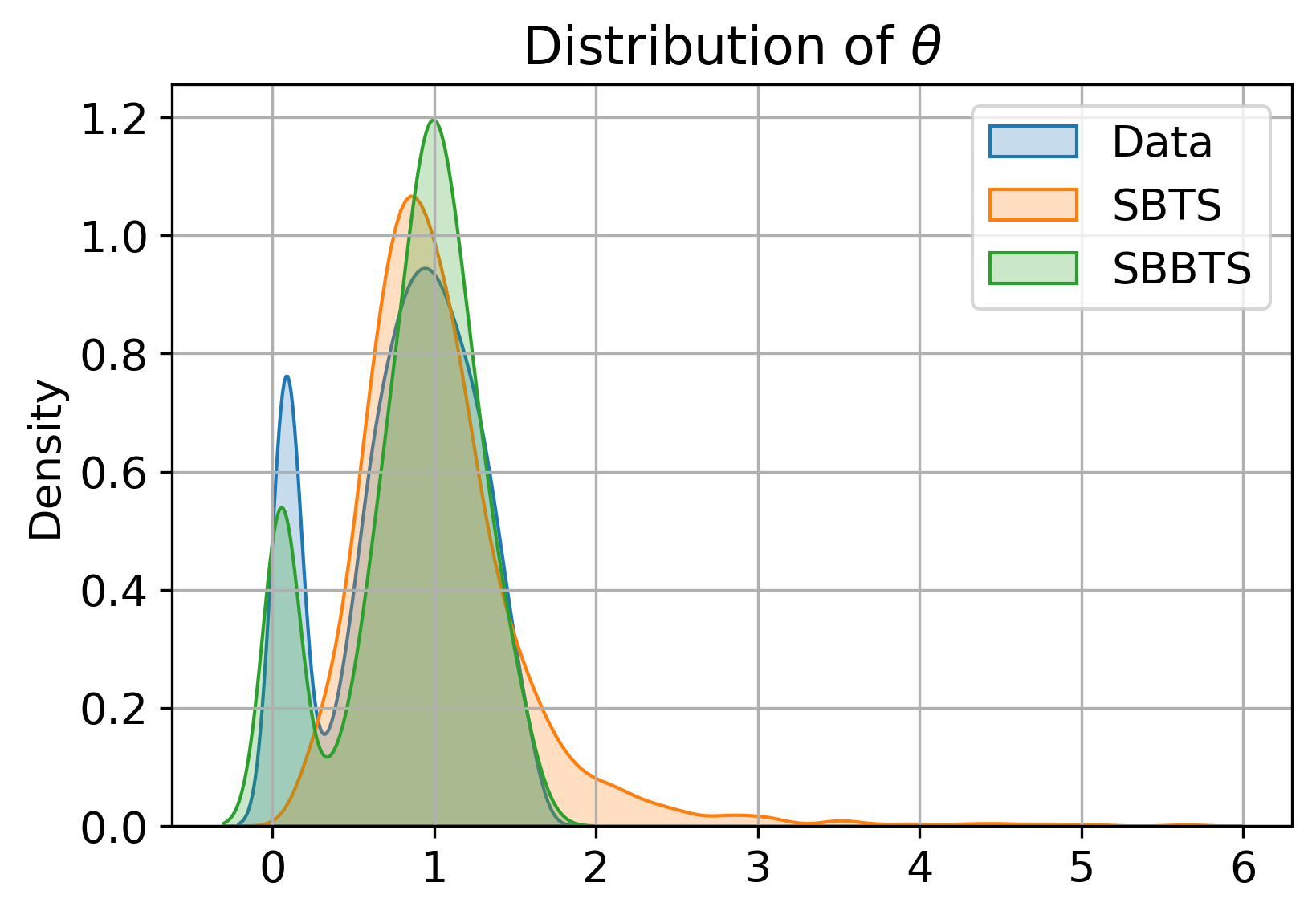}
					\end{subfigure}
					\hfill
					\begin{subfigure}{0.3\textwidth}
						\includegraphics[width=\linewidth]{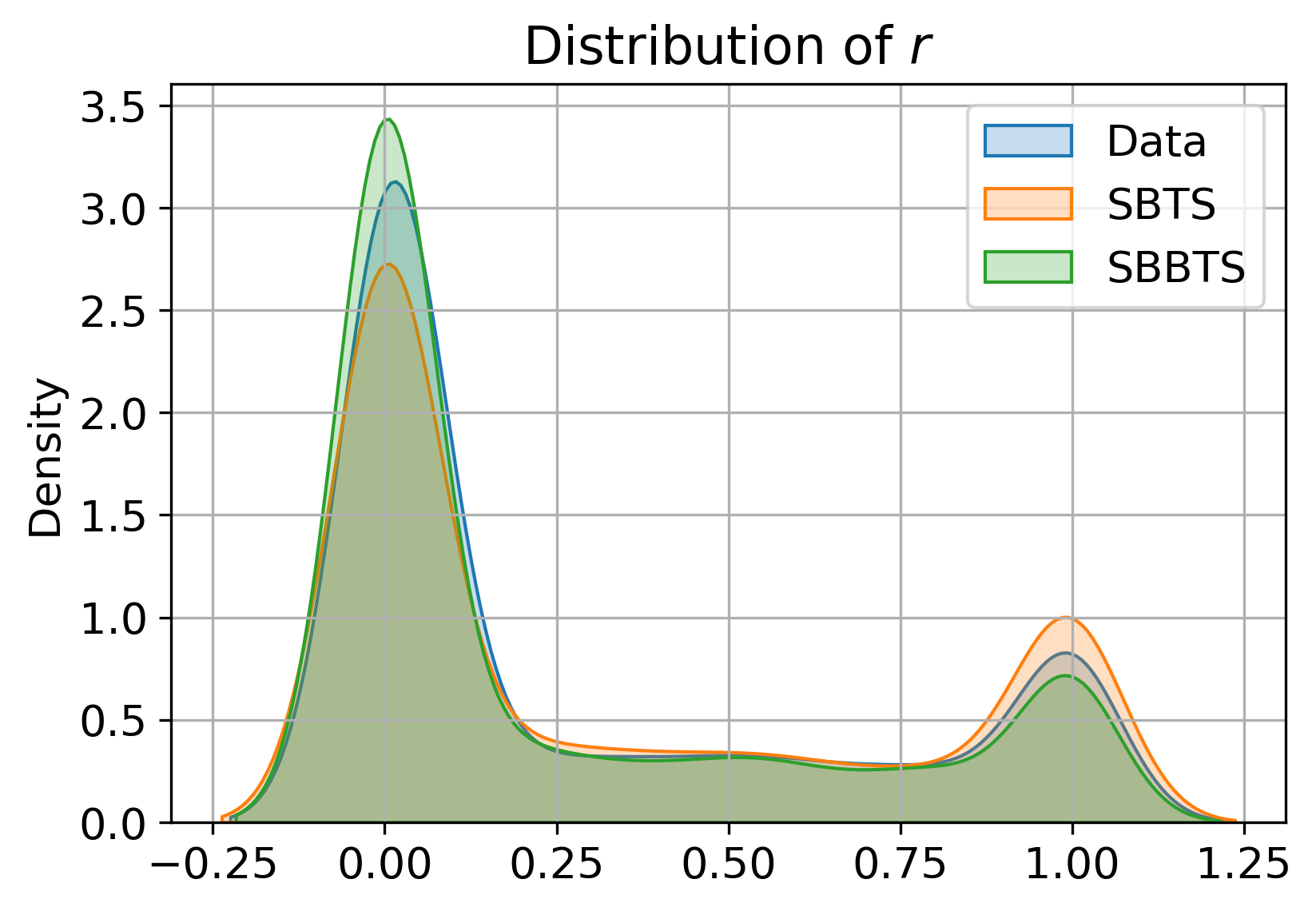}
					\end{subfigure}
					
					\caption{Distribution of estimated Heston parameters using MLE. We show in blue, orange and green the density respectively from the data, SBTS, and SBBTS samples.}
					\label{fig:heston_res}
				\end{figure}
				
				Figure~\ref{fig:heston_res} shows that the SBBTS model more accurately captures the full real range for all parameters and aligns well with the real data distribution. In contrast, the previous SBTS model failed to reproduce the “vol of vol” $\xi$ and the correlation $\rho$. This discrepancy is due to the condition $\P \ll \W$ in the SB framework ---but not in SBB---, which fixes the quadratic variation of the generated paths and precludes stochastic volatility and correlated noise. Consequently, diffusion-driven
				parameters ($\xi$, $\rho$) cannot be faithfully encoded and are projected onto an effective average, yielding a concentrated distribution around the center of the parameter range, while drift-related parameters ($\kappa$, $\theta$, $r$) remain identifiable and well recovered.

				\subsection{Data Augmentation for Time Series Forecasting}
				
				In this part, we evaluate the impact of synthetic time series data on a real-world forecasting task. Additional details can be found in Appendix~\ref{subec:data_augmentation_details}.
				
				\subsubsection{Problem Definition}
				
				In this part, we focus on time series forecasting. Let $X = (X_{t_0}, \cdots, X_{t_{N-1}}) \in (\R^d)^N $, with $d$ the number of instruments, be a time series of daily stock returns. The goal is to predict the probability that the sign of the next daily return is positive. Hence, the predictive model produces an output $\Psi_{\theta}(X_{t_0:t_{N-1}}) = \hat{p}_{t_N} \in [0,1]$, representing the estimated probability that the next return is positive. Since financial returns are mostly noise, we generally expect $\hat{p}_t$ to be close to $0.5$. The objective is to capture any predictive signal - often referred to as \textit{alpha} - which can be expressed as the deviation $|\hat{p}_t - \mathrm{sign}(x_t)|$, where $\mathrm{sign}(X_{t_N}) \in \{0,1\}$ denotes the true direction of the next return. As this is a binary classification problem, the model is trained using the Binary Cross-Entropy loss.
				
				\subsubsection{Predictive Model: TabICL}
				
				For these experiments, we used TabICL \cite{qu2025tabicl}, a transformer‑based tabular foundation model that  achieved state-of-the-art results on TabArena Benchmark \cite{erickson2025tabarenalivingbenchmarkmachine}. It has been pre‑trained exclusively on synthetic datasets, a design choice that mirrors the synthetic‑only training paradigm central to our experiment. Note that TabICL operates in a zero‑shot manner: the original weights released by the authors—used directly for inference without any additional fine‑tuning—will be referred to below as \textit{Zero-Shot}. While \cite{garg2025realtabpfnimprovingtabularfoundation} demonstrated that adding a real‑data fine-tuning stage enhances performance, our work maintains the purely synthetic training regime to investigate how far synthetic augmentation alone can drive accurate prediction of daily return direction.
				
				\subsubsection{Data}
				
				In these experiments, we use daily stock returns from the S\&P~500 over the period
				from 2010-01-05 to 2021-12-31. The dataset consists of $433$ tradable instruments and is
				sourced from \cite{cetingoz2025syntheticdataportfoliosthrow}. The data are split into a
				training set spanning 2010-01-05 to 2018-12-31, a validation set from 2019-01-01 to
				2020-06-30, and a test set from 2020-07-01 to 2021-12-31. Since TabICL operates on tabular data, the time series are transformed into feature
				representations. These features are constructed both independently for each
				instrument and jointly to capture cross-sectional dependencies, using a maximum
				lookback window of $252$ days, corresponding to approximately one trading year. \\
				
				Note that, in order to generate the full set of $433$ stocks, we adopt the dimensionality reduction approach proposed in \cite{cetingoz2025syntheticdataportfoliosthrow}, which
				combines principal component analysis (PCA) with clustering techniques.
				
				\subsubsection{Metrics}
				
				In order to evaluate the predictive power on the model and the impact of synthetic data, we are using metrics that can be split into two dimensions. \\
				
				\textbf{Classification metrics }
				
				\begin{enumerate}
					\item \bef{Accuracy: }  First, we convert the predicted probability $\hat{p}_{t_N}$ into a binary predicted sign using the rule
					\begin{equation*}
						\hat{\mathrm{sign}}(X_{t_N}) =
						\begin{cases}
							1, & \text{if } \hat{p}_{t_N} \geq 0.5, \\
							0, & \text{otherwise.}
						\end{cases}
					\end{equation*}
					The classification accuracy is then computed as $  \text{Accuracy} = \frac{1}{M} \sum_{m=1}^{M} 
					1_{\{\hat{\mathrm{sign}}(X_{t_m}) = \mathrm{sign}(X_{t_m})\}}$, where $M$ denotes the number of samples in the evaluation set. 
					
					\item \bef{Log Loss: } It is defined as 
					\[
					L(X, p) = - \frac{1}{M} \sum_{m=1}^{M} \Big[ \mathrm{sign}(X_{t_m}) \log(\hat{p}_{t_m}) + (1 - \mathrm{sign}(X_{t_m})) \log(1 - \hat{p}_{t_m}) \Big] 
					\]
					\item \bef{ROC AUC Score: } It measures how well a model ranks positive instances higher than negative ones, with $1$ being perfect ranking and $0.5$ being random. 
					
				\end{enumerate} 
				
				\vspace{1em}
				
				\bef{Financial Metrics}
				
				\begin{enumerate}
                    \item \textbf{Daily PnL: } For each day, we compute the position vector 
                    $\mathbf{w}_{t_m} = 2 \times \hat{\mathbf{p}}_{t_m} - \mathbf{1} \in [-1, 1]^d$,
                    where $\hat{\mathbf{p}}_{t_m} \in [0,1]^d$ is the vector of predicted probabilities 
                    of a positive return across all $d$ instruments. The daily PnL is then
                    \[
                    \text{PnL}_{t_m} = \frac{1}{d}\, \mathbf{w}_{t_m}^\top \mathbf{R}_{t_m},
                    \]
                    with $\mathbf{R}_{t_m} \in \mathbb{R}^d$ being the vector of true returns at time 
                    $t_m$ across all instruments. \textit{Note that we assume no transaction cost.}
					
					\item \bef{PnL Standard Deviation: } The standard deviation of the average daily PnL:
					\[
					\sigma_\text{PnL} = \sqrt{\frac{1}{M-1} \sum_{m=1}^{M} \big(\text{PnL}_{t_m} - \overline{\text{PnL}}\big)^2}.
					\]
					
					where $\overline{\text{PnL}}$ is the average daily PnL.
					
					\item \bef{Sharpe ratio: } Defined as the annualized ratio of the average daily PnL to its standard deviation:
					\[
					\text{Sharpe ratio} = \frac{\overline{\text{PnL}}}{\sigma_\text{PnL}} \sqrt{252}.
					\]
				\end{enumerate}

    \subsubsection{Results}
    \label{subsec:results_forecasting}
    
    In this section, we assess the impact of synthetic data generated by \textsc{SBBTS} on downstream forecasting performance.
    All reported metrics are averaged over \textbf{5 independent random seeds}.
    For each metric, we report the mean across seeds, while the vertical error bars represent the corresponding standard deviation.
    
\paragraph{Overall comparison on the test set.}
    Table~\ref{tab:test_comparison} reports the predictive and financial performance of TabICL on the test set under different
    training regimes: zero-shot inference, training with real data only, and training with augmented synthetic \textsc{SBBTS} samples only. In the latter setting, we used $200$ times more synthetic paths than in the real dataset. Results are averaged over 5 independent random seeds; standard deviations are reported in parentheses.

    To verify that the gains obtained with \textsc{SBBTS} are not merely due to injecting additional randomness, we also compare 
    \textsc{SBBTS}-based augmentation with a naive noise-based augmentation strategy.
    Specifically, for each real sample $X$, we generate $p$ additional samples of the form
    \[
    \tilde{X}^{(p)} = X + \lambda \, \varepsilon^{(p)}, \qquad \varepsilon^{(p)} \sim \mathcal{N}(0, \sigma_X^2),
    \]
    with $\lambda = 0.5$.

\begin{table}[h]
    \centering
    \caption{test set performance of TabICL under different training configurations.
    Results are averaged over 5 seeds; standard deviations are reported in parentheses.
    $\downarrow$ indicates lower is better, $\uparrow$ indicates higher is better.
    Best result per row is \textbf{bold}.}
    \label{tab:test_comparison}
    \setlength{\tabcolsep}{6pt}
    \begin{tabular}{lcccc}
        \toprule
        \textbf{Metric} & \textbf{Zero-Shot} & \textbf{Real + Noise} & \textbf{Real} & \textbf{SBBTS} \\
        \midrule
        \multicolumn{5}{l}{\textit{Classification metrics}} \\[2pt]
        Accuracy ($\uparrow$)
            & $0.494$ & $0.518_{\,(0.008)}$ & $0.521_{\,(0.006)}$ & $\mathbf{0.532}_{\,(\mathbf{0.005})}$ \\[3pt]
        Log Loss ($\downarrow$)
            & $0.756$ & $0.695_{\,(0.004)}$ & $0.693_{\,(0.003)}$ & $\mathbf{0.691}_{\,(\mathbf{0.002})}$ \\[3pt]
        ROC AUC ($\uparrow$)
            & $0.486$ & $0.494_{\,(0.012)}$ & $0.497_{\,(0.008)}$ & $\mathbf{0.521}_{\,(\mathbf{0.007})}$ \\
        \midrule
        \multicolumn{5}{l}{\textit{Financial metrics}} \\[2pt]
        Avg Daily Return (\%) ($\uparrow$)
            & $-0.020$ & $0.086_{\,(0.035)}$ & $0.112_{\,(0.020)}$ & $\mathbf{0.143}_{\,(\mathbf{0.015})}$ \\[3pt]
        Std Daily Return (\%) ($\downarrow$)
            & $\mathbf{0.103}$ & $0.105_{\,(0.003)}$ & $0.110_{\,(\mathbf{0.002})}$ & $0.108_{\,(\mathbf{0.002})}$ \\[3pt]
        Sharpe Ratio ($\uparrow$)
            & $-0.254$ & $1.300_{\,(0.44)}$ & $1.613_{\,(0.33)}$ & $\mathbf{2.113}_{\,(\mathbf{0.20})}$ \\
        \bottomrule
    \end{tabular}
\end{table}

    As shown in Table~\ref{tab:test_comparison}, augmenting the training set with \textsc{SBBTS}-generated synthetic data
    consistently improves both classification and financial metrics compared to the zero-shot baseline and the real-data-only setting.
    In particular, we observe systematic gains in ROC AUC and Sharpe ratio, indicating that the model captures more informative
    ranking signals and translates them into improved risk-adjusted returns. Furthermore, white noise augmentation fails to yield consistent gains across metrics—and in some cases degrades performance—whereas \textsc{SBBTS}-based augmentation leads to clear and stable improvements across seeds. 
This indicates that SBBTS captures meaningful temporal and cross-sectional structure, rather than merely injecting additional noise. \\

    \begin{figure}[t]
        \centering
        \includegraphics[width=\columnwidth]{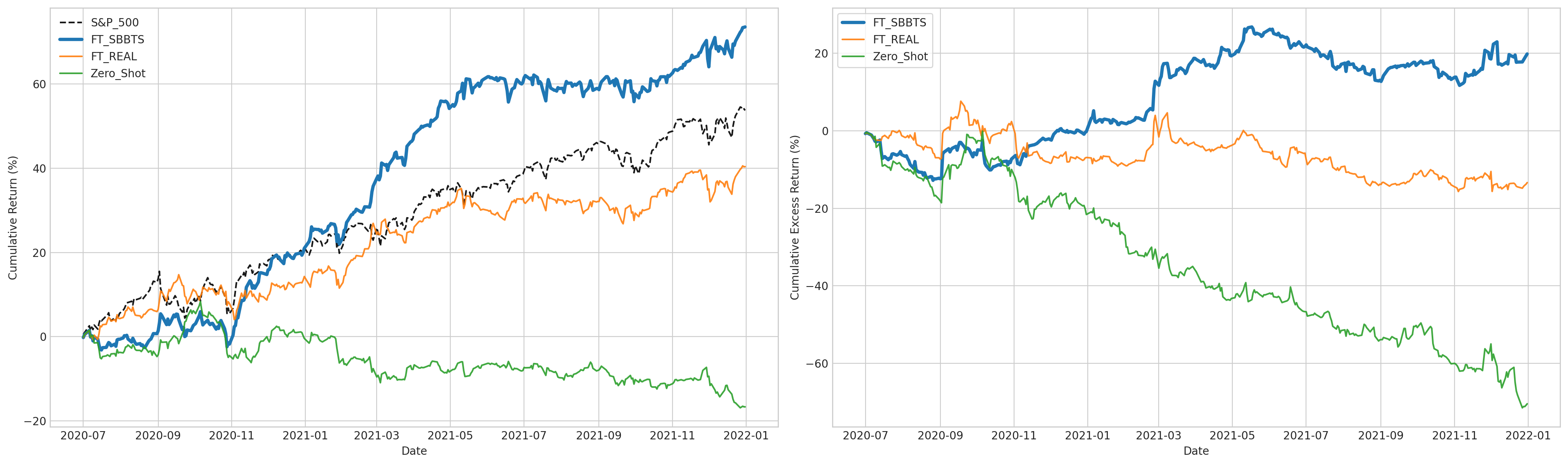}
        \caption{Cumulative return (left) and cumulative excess return (right) on the test period, with S\&P 500 index baseline.}
        \label{fig:cum_returns}
    \end{figure}

    Figure~\ref{fig:cum_returns} provides a time-series view of the trading performance.
    The model trained with \textsc{SBBTS} synthetic data delivers the highest cumulative return and maintains consistently positive excess returns throughout most of the test period.
    By contrast, the zero-shot model shows a persistent deterioration, while the real-data-only model achieves moderate but less stable gains.
    These results confirm that \textsc{SBBTS} augmentation not only improves pointwise predictive metrics but also translates into economically meaningful and more robust out-of-sample performance.

    Note that the objective is not to design the best possible trading strategy, but rather to assess the impact of synthetic data augmentation on model training. In this context, the fact that a simple toy strategy (without transaction cost) already outperforms the baseline when trained with synthetic data is encouraging.

    \paragraph{Effect of the amount of synthetic data.}
    To further analyze the role of synthetic data, Figure~\ref{fig:test_amount} reports the Log Loss and Sharpe ratio
    as a function of the number of synthetic paths used during training.
    This experiment allows us to assess how performance scales with the amount of generated data.

    \begin{figure}[htbp]
        \centering
        \includegraphics[width=\columnwidth]{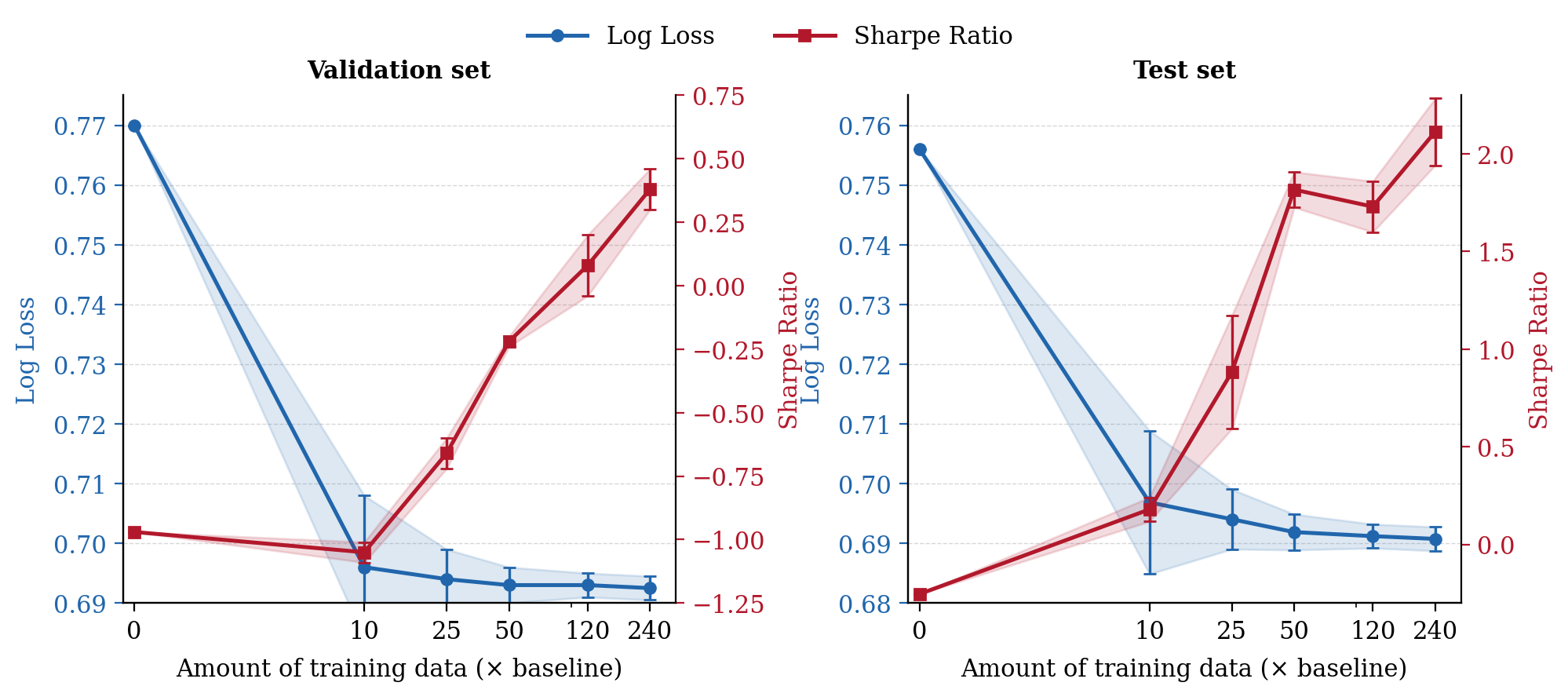}
        \caption{Validation (\textit{left}) and test set (\textit{right}) performance as a function of the amount of synthetic data generated by \textsc{SBBTS}.
    	Results are averaged over 5 seeds; error bars indicate one standard deviation.}
        \label{fig:test_amount}
    \end{figure}

    Figure~\ref{fig:test_amount} shows that performance improves as the amount of synthetic data increases on both validation and test set, up to a moderate regime where gains begin to saturate. This suggests that \textsc{SBBTS} effectively enriches the training distribution by exposing the predictive model to a broader set of plausible market scenarios, while additional synthetic data beyond this regime does not introduce instability or overfitting.

    Overall, these results provide strong empirical evidence that \textsc{SBBTS} generates synthetic time series that preserve
    and amplify predictive signal, making them well suited for data augmentation in financial forecasting tasks.
    Additional results on the validation set, together with a discussion of the statistical significance of the Sharpe ratio,
    are provided in Appendix~\ref{subsec:sup_res_trad}.

\section{Conclusion}
This paper introduced the Schrödinger Bass Bridge for Time Series (SBBTS), a novel generative framework that unifies Schrödinger Bridge and Bass martingale principles to jointly calibrate both drift and volatility in time series generation. By decomposing the problem into a sequence of semimartingale optimal transport steps, SBBTS provides an efficient and scalable algorithm that overcomes the volatility calibration limitations of traditional Schrödinger Bridge models.

Empirical results demonstrate the practical value of SBBTS across multiple domains. In synthetic experiments with the Heston model, SBBTS successfully recovers stochastic volatility and correlation parameters that previous methods failed to capture. In financial forecasting applications, SBBTS-generated synthetic data consistently enhances model performance, improving both classification metrics and risk-adjusted returns when used for data augmentation.

\subsection*{Limitations and Future Work}
While SBBTS demonstrates encouraging results, certain limitations warrant discussion. Notably, the model's behavior is influenced by the regularization parameter \(\beta\), whose optimal selection currently lacks a systematic criterion—although practical guidelines recommend avoiding excessively small values. The large-\(\beta\) approximation adopted in this work aligns with typical financial time scales, but the more general iterative scheme from \cite{alouadi2026lightsbbmbridgingschrodingerbass} could be adapted to our framework. On the theoretical side, although Algorithm~\ref{alg:sbb_algo_beta_large} converges consistently within \(K=5\) iterations in our experiments, formal convergence guarantees remain an open question.

These considerations highlight several promising research directions. Future work could develop principled methods for \(\beta\) calibration, extend the framework to incorporate jump-diffusion dynamics or irregularly sampled observations, and establish rigorous convergence proofs for the proposed training algorithm. Such advances would further solidify the theoretical foundations of SBBTS and broaden its applicability to more complex temporal data structures.

\subsection*{Acknowledgements}

This work was conducted in collaboration with the CMAP Laboratory at École Polytechnique, the LPSM Laboratory and BNP Paribas CIB Global Markets. We thank Baptiste Barreau (BNPP) and Charles-Albert Lehalle (CMAP) for their helpful discussions and feedback on early drafts of this work.

%				\bibliographystyle{plain}

%				\bibliography{biblioSBB}
				
%				\newpage
				
				\appendix
                \renewcommand{\theequation}{\arabic{equation}}
				
				\section{Reference Volatility and Proof of Theorem \ref{propreduc}}
				
				\subsection{Reference Volatility}
				
				Instead of using the identity matrix $I_d$ as reference volatility in \eqref{defJ},  one could  choose  the covariance matrix  of the time series distribution $\mu$ over the interval $(t_i,t_{i+1}]$, namely: 
				\begin{align}
					{\rm Var}_{\mu,i} & := \; \E_\mu \Big[  \big( \Delta X_{t_i} - \E_\mu[\Delta X_{t_i}] \big) \big( \Delta X_{t_i} - \E_\mu[\Delta X_{t_i}] \big)\trans \Big], \quad \Delta X_{t_i} := X_{t_{i+1}} - X_{t_i},   
				\end{align}  
				for $i$ $=$ $0,\ldots,n-1$. This covariance matrix can  be estimated from samples of $\mu$.  Then, we define a sequence of constant matrices $\bar \sigma_i$ in $\R^{d\times d}$, $i$ $=$ $0,\ldots,n-1$,  by
				\begin{align}
					\bar \sigma_i\bar  \sigma_i\trans & = \; {\rm Var}_{\mu,i}, 
				\end{align}
				where $\bar \sigma_i \bar \sigma_i\trans > 0$ (i.e., positive definite) and form a piecewise-constant deterministic volatility $\bar\sigma(t)$. In other words, this reference deterministic volatility  is calibrated to the time series variance between each observation time interval, and 
				we  study  a criterion in the form 
				\begin{align} \label{SBBbis} 
					{\rm SBB}(\P) &= \; \E_\P \Big[ \frac{1}{2} \int_0^T  \|\alpha_t\|^2_{\bar\Sigma^{-1}}  + \beta \|\sigma_t - \bar\sigma(t)\|^2 \d t \Big],   
				\end{align}
				where  we set $\bar\Sigma$ $=$ $\bar\sigma\bar\sigma\trans$.  
				
				\subsection{Proof of Theorem \ref{propreduc} }
				\label{sec:proof_theorem}

\begin{proof}
	{\it Step 1.} For $\nu\in \Pc((\R^d)^{i+2})$, define
	\begin{align}
		\widetilde V_i(\nu)
		:=
		\inf_{\P\in\Pc(\nu)}
		\E_\P\Big[\int_{t_i}^{t_{i+1}}L(\alpha_t,\sigma_t)\,\d t\Big],
		\qquad i=0,\ldots,n-1,
	\end{align}
	where
	\begin{align}
		L(a,\gamma):=\frac12\Big(|a|^2+\beta|\gamma-I_d|^2\Big).
	\end{align}
	For $x_{0:i}\in(\R^d)^{i+1}$, let
	\begin{align}
		\nu^{x_{0:i}}(\d y_{0:i+1})
		:=
		\delta_{x_{0:i}}(\d y_{0:i})\,\mu_{i+1|0:i}(\d y_{i+1}\mid x_{0:i}),
	\end{align}
	so that, by definition of $V_i$ in \eqref{defVi},
	\begin{align}
		V_i(x_{0:i})=\widetilde V_i(\nu^{x_{0:i}}).
	\end{align}

	Let now $\P\in\Pc(\mu_{i+1})$, and let $(\P_{x_{0:i}})_{x_{0:i}}$ be a regular conditional probability distribution of $\P$ given $X_{t_0:t_i}=x_{0:i}$. Then for $\mu_i$-a.e.\ $x_{0:i}$,
	\begin{align}
		\P_{x_{0:i}}\circ X_{t_0:t_{i+1}}^{-1}=\nu^{x_{0:i}},
	\end{align}
	hence $\P_{x_{0:i}}\in\Pc(\nu^{x_{0:i}})$ for $\mu_i$-a.e.\ $x_{0:i}$. Therefore
	\begin{align}
		\E_\P\Big[\int_{t_i}^{t_{i+1}}L(\alpha_t,\sigma_t)\,\d t\Big]
		&=
		\E_\P\Big[\E_\P\Big[\int_{t_i}^{t_{i+1}}L(\alpha_t,\sigma_t)\,\d t\ \Big|\ X_{t_0:t_i}\Big]\Big] \\
		&=
		\int \E_{\P_{x_{0:i}}}\Big[\int_{t_i}^{t_{i+1}}L(\alpha_t,\sigma_t)\,\d t\Big]\mu_i(\d x_{0:i}) \\
		&\ge
		\int V_i(x_{0:i})\,\mu_i(\d x_{0:i})
		=
		\E_{\mu_i}\big[V_i(X_{t_0:t_i})\big].
	\end{align}
	This implies
	\begin{align}
		\widetilde V_i(\mu_{i+1})
		\ge
		\E_{\mu_i}\big[V_i(X_{t_0:t_i})\big].
		\label{eq:firstineq_compromise}
	\end{align}

	For the converse inequality, fix $\eps>0$. By a standard measurable-selection argument, one may choose a universally measurable family
	$x_{0:i}\mapsto \P^\eps_{x_{0:i}}$ such that, for every $x_{0:i}$,
	\begin{align}
		\P^\eps_{x_{0:i}}\in \Pc(\nu^{x_{0:i}})
		\qquad\text{and}\qquad
		\E_{\P^\eps_{x_{0:i}}}\Big[\int_{t_i}^{t_{i+1}}L(\alpha_t,\sigma_t)\,\d t\Big]
		\le
		V_i(x_{0:i})+\eps.
		\label{eq:epsselector_compromise}
	\end{align}
	Define $\P^\eps\in\Pc$ by
	\begin{align}
		\P^\eps(A):=\int \P^\eps_{x_{0:i}}(A)\,\mu_i(\d x_{0:i}),
		\qquad A\in\Fc.
	\end{align}
	By construction,
	\begin{align}
		\P^\eps\circ X_{t_0:t_{i+1}}^{-1}=\mu_{i+1},
	\end{align}
	so $\P^\eps\in\Pc(\mu_{i+1})$, and
	\begin{align}
		\widetilde V_i(\mu_{i+1})
		&\le
		\E_{\P^\eps}\Big[\int_{t_i}^{t_{i+1}}L(\alpha_t,\sigma_t)\,\d t\Big] \\
		&=
		\int \E_{\P^\eps_{x_{0:i}}}\Big[\int_{t_i}^{t_{i+1}}L(\alpha_t,\sigma_t)\,\d t\Big]\mu_i(\d x_{0:i}) \\
		&\le
		\int V_i(x_{0:i})\,\mu_i(\d x_{0:i})+\eps.
	\end{align}
	Letting $\eps\downarrow0$ and combining with \eqref{eq:firstineq_compromise}, we obtain
	\begin{align}
		\widetilde V_i(\mu_{i+1})
		=
		\E_{\mu_i}\big[V_i(X_{t_0:t_i})\big].
		\label{eq:tildeVmui_compromise}
	\end{align}

	\vspace{1mm}

	\noindent {\it Step 2.} Next, define
	\begin{align}
		\bar V_i
		:=
		\inf_{\P\in\Pc(\mu_i)}
		\E_\P\Big[\int_0^{t_i}L(\alpha_t,\sigma_t)\,\d t\Big],
		\qquad i=0,\ldots,n,
	\end{align}
	and we claim that
	\begin{align}
		\bar V_{i+1}
		=
		\bar V_i+\E_{\mu_i}\big[V_i(X_{t_0:t_i})\big],
		\qquad i=0,\ldots,n-1.
		\label{eq:dynamic_prog_compromise}
	\end{align}

	Let $\P\in\Pc(\mu_{i+1})$. Then
	\begin{align}
		\E_\P\Big[\int_0^{t_{i+1}}L(\alpha_t,\sigma_t)\,\d t\Big]
		&=
		\E_\P\Big[\int_0^{t_i}L(\alpha_t,\sigma_t)\,\d t
		+\int_{t_i}^{t_{i+1}}L(\alpha_t,\sigma_t)\,\d t\Big] \\
		&\ge
		\E_\P\Big[\int_0^{t_i}L(\alpha_t,\sigma_t)\,\d t\Big]
		+\widetilde V_i(\mu_{i+1}) \\
		&\ge
		\bar V_i+\E_{\mu_i}\big[V_i(X_{t_0:t_i})\big],
	\end{align}
	since $\Pc(\mu_{i+1})\subset\Pc(\mu_i)$ and by \eqref{eq:tildeVmui_compromise}. This proves the inequality ``$\ge$'' in \eqref{eq:dynamic_prog_compromise}.

	For the reverse inequality, fix $\eps>0$ and choose $\P^{1,\eps}\in\Pc(\mu_i)$ such that
	\begin{align}
		\E_{\P^{1,\eps}}\Big[\int_0^{t_i}L(\alpha_t,\sigma_t)\,\d t\Big]
		\le
		\bar V_i+\eps.
		\label{eq:prefixeps_compromise}
	\end{align}
	Let $(\P^{1,\eps}_{x_{0:i}})_{x_{0:i}}$ be a regular conditional probability distribution of $\P^{1,\eps}$ given $X_{t_0:t_i}=x_{0:i}$. By a standard pasting argument, using the measurable family $(\P^\eps_{x_{0:i}})_{x_{0:i}}$ from Step~1, one can construct a probability measure $\Q^\eps\in\Pc(\mu_{i+1})$ by concatenating, for each $x_{0:i}$, the prefix law $\P^{1,\eps}_{x_{0:i}}$ on $[0,t_i]$ with the continuation law $\P^\eps_{x_{0:i}}$ on $[t_i,t_{i+1}]$.

	By construction, $\Q^\eps$ has the correct joint marginal $\mu_{i+1}$, and the cost splits as
	\begin{align}
		\E_{\Q^\eps}\Big[\int_0^{t_{i+1}}L(\alpha_t,\sigma_t)\,\d t\Big]
		&=
		\E_{\P^{1,\eps}}\Big[\int_0^{t_i}L(\alpha_t,\sigma_t)\,\d t\Big]
		+
		\int \E_{\P^\eps_{x_{0:i}}}\Big[\int_{t_i}^{t_{i+1}}L(\alpha_t,\sigma_t)\,\d t\Big]\mu_i(\d x_{0:i}) \\
		&\le
		\bar V_i+\eps+\int \big(V_i(x_{0:i})+\eps\big)\mu_i(\d x_{0:i}) \\
		&=
		\bar V_i+\E_{\mu_i}\big[V_i(X_{t_0:t_i})\big]+2\eps,
	\end{align}
	where we used \eqref{eq:prefixeps_compromise} and \eqref{eq:epsselector_compromise}. Since $\Q^\eps\in\Pc(\mu_{i+1})$, this yields
	\begin{align}
		\bar V_{i+1}
		\le
		\bar V_i+\E_{\mu_i}\big[V_i(X_{t_0:t_i})\big]+2\eps.
	\end{align}
	Letting $\eps\downarrow0$, we obtain the reverse inequality in \eqref{eq:dynamic_prog_compromise}.

	Hence \eqref{eq:dynamic_prog_compromise} holds for every $i=0,\ldots,n-1$. We conclude by forward induction on $i$, and by noting that ${\rm SBBTS}(\mu)=\bar V_n$.
\end{proof}

				\section{Neural Network Architecture}
				\label{sec:nn_archi}
				
				We describe the architecture of the model used throughout our experiments, illustrated in Figure~\ref{fig:netw}.  
				
				First, both the time step $t \in \mathbb{R}$ and the current value $Y_t \in \mathbb{R}^d$ are mapped onto a latent space of dimension $d_{\text{model}}$ using an independent Feed Forward Network (FNN). This FNN consists of a linear layer, layer normalization, the SiLU activation function \cite{article_silu}, and a final linear layer.  
				
				The past sequence $Y_{t_0:t_i} \in (\mathbb{R}^d)^{i+1}$ is first embedded into the latent space via a linear layer and then encoded as a vector $c_i \in \mathbb{R}^{d_{\text{model}}}$ using an encoder-only architecture from \cite{NIPS2017_3f5ee243} with one layer. A mask is applied during training to ensure the transformer does not see future time steps.  
				
				Finally, all embedded vectors are concatenated and mapped back to the original space of dimension $d$ using a similar FNN. The output is the estimated drift:  
				\[
				s_{\theta}(t, Y_t, Y_{t_0:t_i}) \simeq \varepsilon \, \nabla_y \log h^*_t(Y_t).
				\]

				\section{Additional Details on Numerical Experiments}
                \label{sec:details_num_exp}
				
				This section provides complementary details on the numerical experiments. All experiments were conducted on a single NVIDIA A100 SXM4 GPU with 40\,GB of memory.
				
				Unless stated otherwise, the parameters used to generate the synthetic time series during both training and inference are summarized in Table~\ref{tab:parameters_exp}.
				
				\begin{table}[!ht]
					\centering
					\begin{tabular}{ccccccccc}
						\toprule
						$K$ & $T$ & $\tilde{T}$ & $n_{\text{epoch}}$ & Batch Size & $lr$ & $d_{\text{model}}$ & $n_{\text{head}}$ & $N^{\pi}$ \\ 
						\midrule
						$5$ & $1$ & $0.99$ & $1000$ & $128$ & $10^{-3}$ & $128$ & $16$ & $50$ \\
						\bottomrule
					\end{tabular}
					\caption{Parameters used in the numerical experiments.}
					\label{tab:parameters_exp}
				\end{table}
				
				Here, $N^{\pi}$ denotes the number of time steps used to simulate the diffusion process~\eqref{DSBY} via the Euler--Maruyama scheme. We also used the Adam optimizer \cite{2015-kingma} to train the neural network.
				Moreover, we follow the same scaling procedure introduced in \cite{Alouadi_2025} (see Section~6).

				\subsection{Heston Process}
				Table~\ref{tab:params_combined_resized} reports the ranges of parameters used to generate the training dataset. We then fit our generative model on this dataset and estimate the parameters using the maximum likelihood estimation (MLE) approach described in \cite{Alouadi_2025}.
				
				\begin{table}[h]
					\centering
					\begin{tabular}{@{}ccccc@{}}
						\toprule
						$\kappa$ & $\theta$ & $\xi$ & $\rho$ & $r$ \\
						\midrule
						$[0.5,\,4]$ & $[0.5,\,1.5]$ & $[0.1,\,0.9]$ & $[-0.9,\,0.9]$ & $[0.01,\,0.1]$ \\
						\bottomrule
					\end{tabular}
					\caption{Parameter ranges used for simulating the Heston process.}
					\label{tab:params_combined_resized}
				\end{table}

				\subsection{Data Augmentation}
                \label{subec:data_augmentation_details}
				
				We provide additional details on the data augmentation experiments in this section.

    \subsubsection{Synthetic data quality assessment}
    
    We evaluate the quality of the generated synthetic time series by comparing several statistical properties of the real and synthetic datasets, focusing on both temporal and cross-sectional structures. 
    
    First, we assess in Figure~\ref{fig:clusters_correl} the temporal dependence structure within each cluster by comparing the autocorrelation functions of returns and squared returns. Overall, the synthetic time series successfully reproduce the main autocorrelation patterns observed in the real data. The autocorrelation curves of the synthetic series appear smoother than those of the real data. This effect is mainly due to the averaging behavior of the neural network approximation: by learning a smooth estimate of the underlying dynamics through a mean-squared training objective, the model filters out high-frequency sampling noise present in the empirical autocorrelation.
    
    \begin{figure}[htbp]
        \centering
        \begin{subfigure}[t]{\textwidth}
            \centering
            \includegraphics[width=\textwidth]{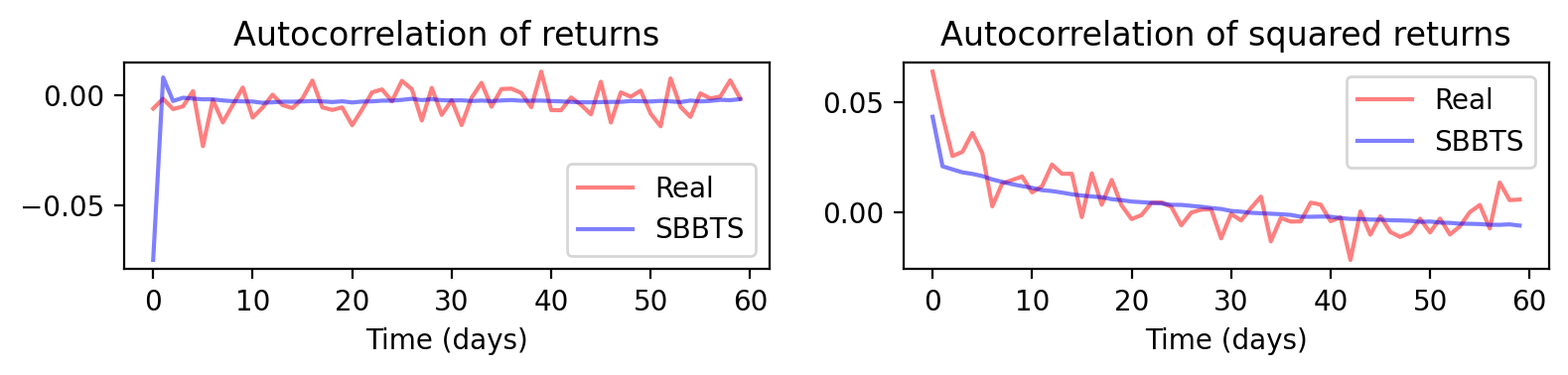}
            \caption{Cluster $1$}
            \label{fig:nm}
        \end{subfigure}
        \vspace{0.5em}
        \begin{subfigure}[t]{\textwidth}
            \centering
            \includegraphics[width=\textwidth]{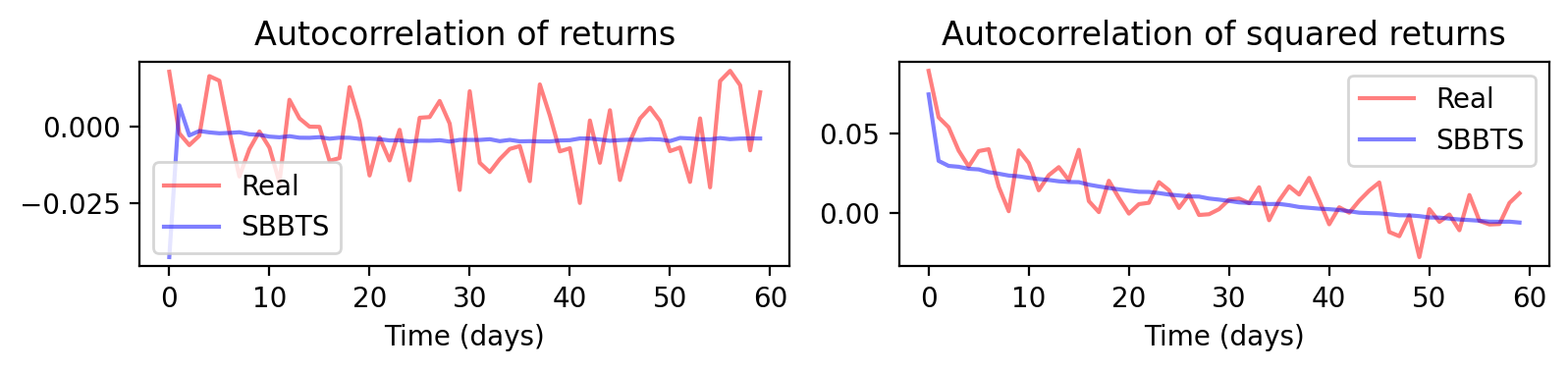}
            \caption{Cluster $2$}
            \label{fig:m8g}
        \end{subfigure}
        \vspace{0.5em}
        \begin{subfigure}[t]{\textwidth}
            \centering
            \includegraphics[width=\textwidth]{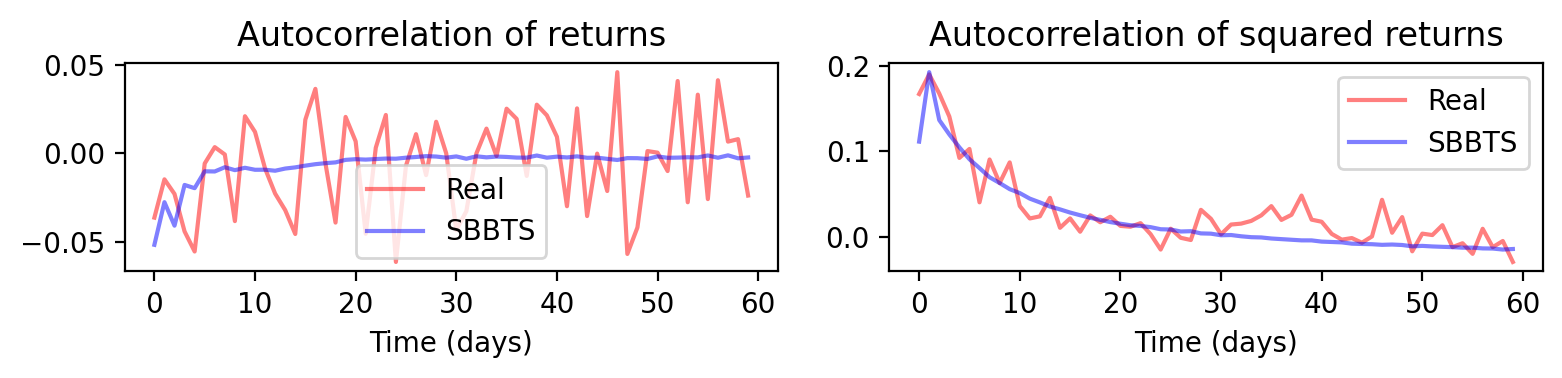}
            \caption{Cluster $3$}
            \label{fig:n8g}
        \end{subfigure}
        \caption{Autocorrelation functions of returns and squared returns for real and synthetic data across clusters.}
        \label{fig:clusters_correl}
    \end{figure}
    
    Next, we compare the marginal distributions of the factors within each cluster. Figure~\ref{fig:clusters_bins} illustrates that the synthetic samples closely match the empirical distributions of the real data. This agreement indicates that the model accurately captures the distributional characteristics of the latent factors across clusters.
    
    \begin{figure}[htbp]
        \centering
        \ \includegraphics[width=\columnwidth]{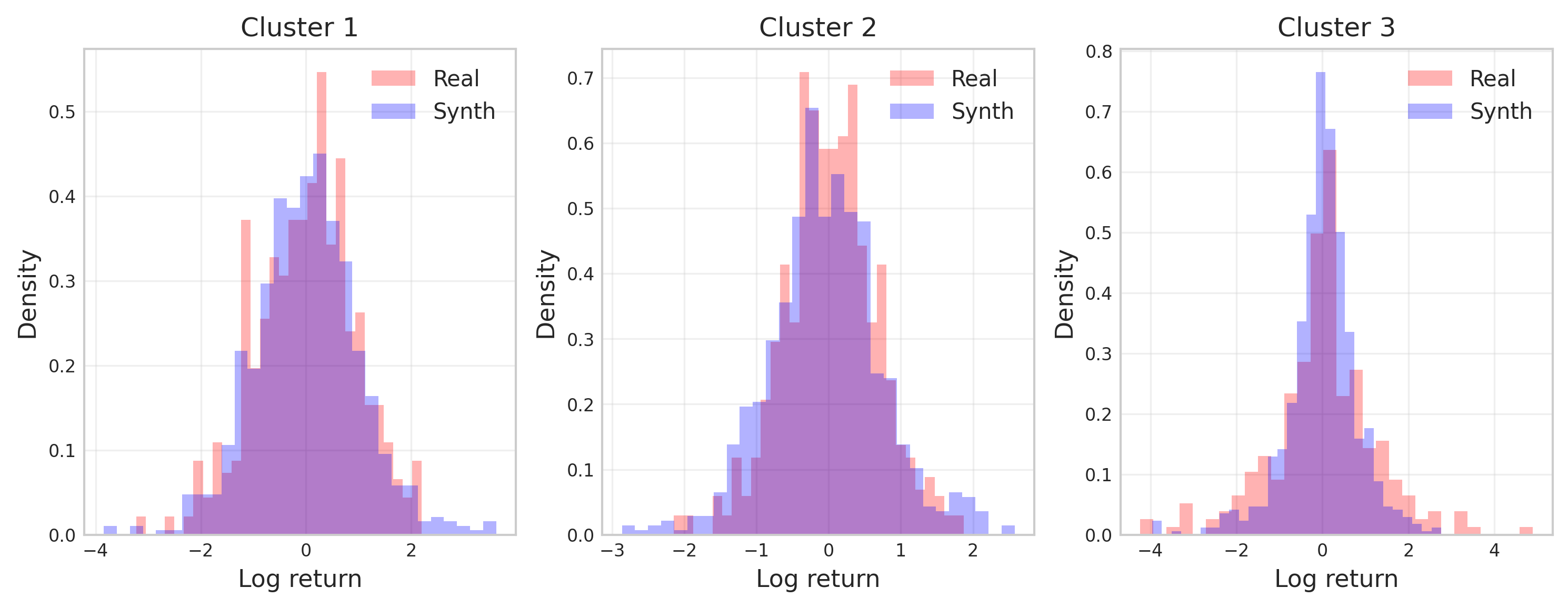}
        \caption{Distribution of cluster factors in real and synthetic data.}
        \label{fig:clusters_bins}
    \end{figure}
    
    Moreover, we examine the cross-sectional dependence structure by comparing the correlation matrices of returns computed from real and synthetic datasets. As shown in Figure~\ref{fig:corr_mat_sp}, the synthetic data preserve the main correlation patterns observed in the real market, confirming that the model is able to replicate not only temporal dynamics but also cross-asset relationships.
    
    \begin{figure}[t]
        \centering
        \includegraphics[width=\columnwidth]{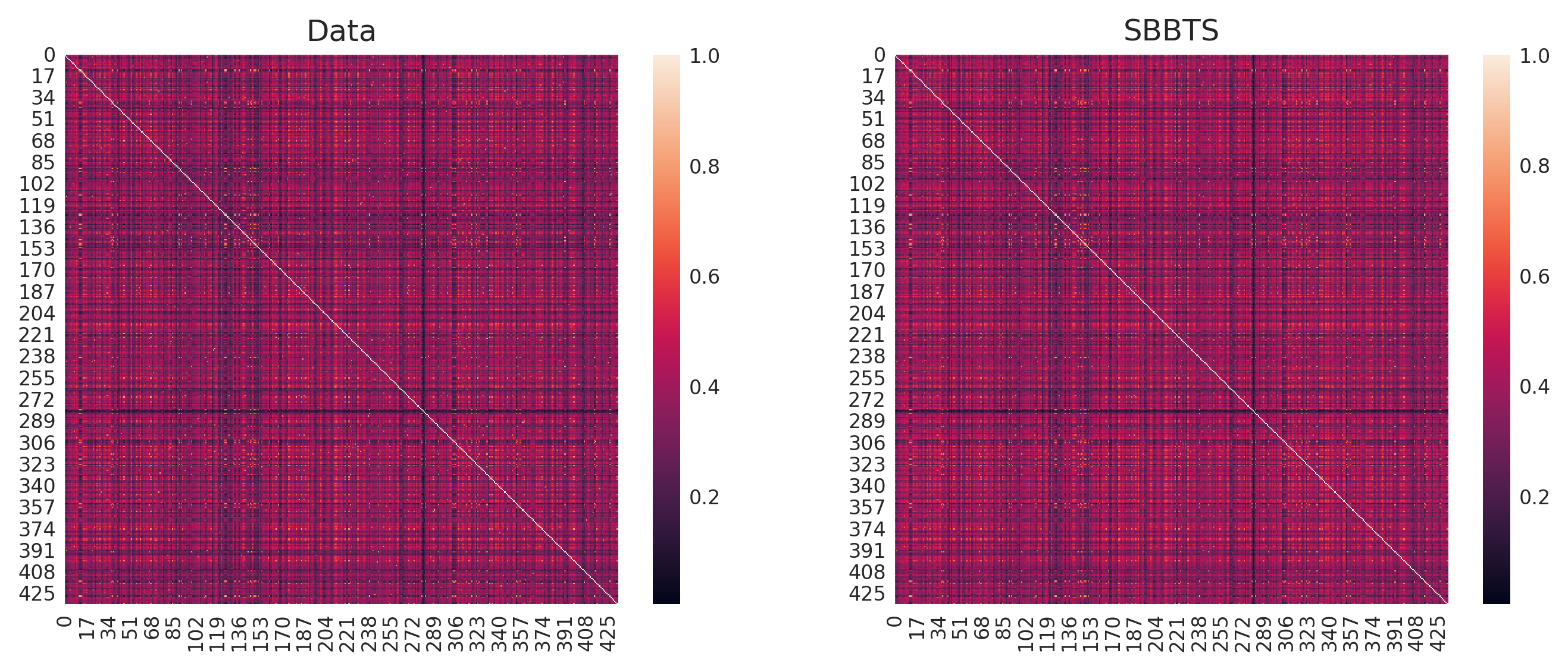}
        \caption{Correlation matrices of returns for real and synthetic data.}
        \label{fig:corr_mat_sp}
    \end{figure}

    Finally, Table~\ref{tab:stat_metric} reports a quantitative comparison of tail-risk statistics averaged across all instruments, using the SBTS framework \cite{hamdouche2023generative} as a benchmark. We evaluate the Value at Risk (VaR) and Expected Shortfall (ES) at the $95\%$ and $99\%$ confidence levels, together with the annualized return and annualized standard deviation. These metrics jointly characterize the tail behavior and the overall risk--return profile of the generated series relative to real data.
    
    \begin{table}[H]
    \centering
    \begin{tabular}{cccc}
    \toprule
     & Real & SBBTS & SBTS \\
    \midrule
    $\text{VaR}_{99\%}$ (\%) & $3.60$ & $3.57$ & $3.49$ \\
    $\text{VaR}_{95\%}$ (\%) & $2.11$ & $2.19$ & $2.17$ \\
    $\text{ES}_{99\%}$ (\%)  & $4.65$ & $4.44$ & $4.36$ \\
    $\text{ES}_{95\%}$ (\%)  & $3.15$ & $3.18$ & $3.07$ \\
    Ann.\ Ret (\%)           & $19.02$ & $16.31$ & $14.68$ \\
    Ann.\ Std (\%)           & $24.22$ & $24.38$ & $23.06$ \\
    \bottomrule
    \end{tabular}
    \caption{Comparison of tail-risk metrics (VaR, ES) and annualized performance statistics averaged across instruments.}
    \label{tab:stat_metric}
    \end{table}

				\subsubsection{TabICL Setup}
				TabICL, when applied to a dataset of size $\mathbb{R}^{n\times m}$, employs a column‑then‑row attention mechanism whose computational complexity scales as $\mathcal{O}(n^{2}m + n)$, thereby imposing substantial constraints on both execution time and GPU‑VRAM usage; consequently, we are compelled to make particular design choices in the continuation pre‑training phase, which are described below.
				\paragraph{Training Framework : }
				
				In the remainder of the paper we refer to each individual forecast problem as an \textbf{episode}. An episode is defined by a contextual window of $n_{context} =22 \text{ days}$  $[d_{r},d_{r+1}, ..., d_{r+n_{context}}]$, and the model is required to predict, in a single forward pass, the return of the next day $d_{r+n_{context}+1}$, for all 433 instruments simultaneously.
				
				We denote by a \textbf{path} a synthetic return matrix of dimension $\mathcal{R}\in\mathbb{R}^{252\times 433},$
				generated with the SBBTS model. Given a context length $n_{\text{context}}$, a single path yields $252 - n_{\text{context}}$ distinct forecasting \emph{episodes}.
				
				Our goal is to expose TabICL to the maximum possible diversity of synthetic episodes during the continuation pre‑training phase, while respecting a fixed
				computational and time budget.  Consequently, at every training epoch we sample $n_{\text{path}}$ paths (e.g.,$n_{\text{path}}=40$ or $n_{\text{path}}=200$) and
				process only a randomly\footnote{The sampling is performed by selecting contiguous blocks of days rather than individual days.  For each path we first draw a block length (e.g., 5 days, 22 days, etc.), then uniformly choose a starting index and take the whole block.  This yields episodes of varying temporal extent while preserving the natural correlation structure of the returns.} selected $50\%$ of the episodes contained in those paths.  
				Formally, for each epoch we draw a set 
				$
				\mathcal{P}_{\text{epoch}}=\{ \mathcal{R}^{(1)},\dots,\mathcal{R}^{(n_{\text{path}})}\},
				$
				and perform the forward–backward pass on a subset $\mathcal{E}_{\text{epoch}}$ of episodes defined as :
				$$
				\mathcal{E}_{\text{epoch}}\subseteq\bigcup_{k=1}^{n_{\text{path}}}
				\Bigl\{\,\text{episodes in }\mathcal{R}^{(k)}\,\Bigr\},
				\qquad |\mathcal{E}_{\text{epoch}}| = 0.5\,
				\bigl( n_{\text{path}}\,(252-n_{\text{context}}) \bigr)
				$$
				This sampling scheme yields a rich, ever‑changing training distribution while
				keeping the per‑epoch computational cost tractable. As suggested in \cite{qu2025tabicl}, to restore a form of permutation invariance, we shuffle feature order across each epoch.
				
				We use the log loss, weighting each observation by the normalized absolute value of its realized return, and
                the  AdamW \cite{loshchilov2019decoupledweightdecayregularization} optimizer with a $lr=3e-7$ learning rate as suggested in \cite{garg2025realtabpfnimprovingtabularfoundation} to avoid catastrophic forgetting and accumulate the gradient over 32 episodes.
				\paragraph{Evaluation Framework :}
				
				We employ early‑stopping on the real‑world SP500 validation split, which spans  377 trading days.  Given our context length $n_{\text{context}}$, this validation window thus yields $377 - n_{\text{context}} = 355$ forecasting \emph{episodes}. Consequently, after each epoch we compute the validation log-loss (and auxiliary metrics) over all $355$ episodes and stop training when the performance ceases to improve according to a patience of $6$ epochs.  The final model is then evaluated on the held‑out test set by processing every available episode within the real‑world SP500 test split with the identical $n_{\text{context}}$ window, ensuring a fair and consistent comparison across all experimental conditions.

				\subsubsection{Feature Engineering from Raw Returns}
				
				We convert a matrix of daily
				returns $\mathbf{R}\in(\mathbb{R}^{d})^{N}$  into a
				tabular dataset suitable for TabICL.  Each row corresponds to a single
				instrument on a single day and contains a set of handcrafted statistics that
				aim to produce an approximately i.i.d. representation of the underlying
				financial process.
				
				\begin{itemize}
					\item \texttt{feature.return\_t-1\_market} : the marketwide lag‑1 return $\tilde R_{t-1}= \frac{1}{d}\sum_{i=1}^{d} R_{t-1,i}.$
					
					\item \texttt{feature.cum\_ret\_h1} : $h_1$‑day cumulative return for the instrument $i$,
					$\text{cum\_ret}_{t,i}^{(h_1)}=\sum_{k=0}^{h_1 -1} R_{t-k,i}.$
					
					\item \texttt{feature.vol\_h1} : volatility of the last $h_1$ returns,
					$$\text{vol}_{t,i}^{(h_1)}=
					\sqrt{\frac{1}{h_1-1}\sum_{k=0}^{h_1-1}
						\bigl(R_{t-k,i}-\bar R_{t,i}^{(h_1)}\bigr)^{2}},\qquad
					\bar R_{t,i}^{(h_1)}=\frac{1}{h_1}\sum_{k=0}^{h_1-1}R_{t-k,i}.$$
					
					\item \texttt{feature.ret\_t-1\_zscore\_h} : $z$‑score of the lag‑1 return of instrument i,
					$$z_{t,i}^{(h)}=
					\frac{R_{t-1,i}-\mu_{t,i}^{(h)}}{\sigma_{t,i}^{(h)}},\qquad
					\mu_{t,i}^{(h)}=\frac{1}{h}\sum_{k=0}^{h-1}R_{t-k,i},\quad
					\sigma_{t,i}^{(h)}=
					\sqrt{\frac{1}{h-1}\sum_{k=0}^{h-1}
						\bigl(R_{t-k,d}-\mu_{t,d}^{(h)}\bigr)^{2}}.$$
					
					\item \texttt{feature.mkt\_cumret\_h} : cumulative market return over the
					past $h$ days,
					$$\text{mkt\_cumret}_{t}^{(h)}=\sum_{k=0}^{h-1}\tilde R_{t-k}.$$
					
					\item \texttt{feature.mkt\_vol\_h} : market volatility computed analogously
					to \texttt{feature.vol\_h} but on the market series $\tilde R_{t}$.
					
					\item \texttt{feature.mkt\_mean\_h } : simple moving average of the market
					return,
					$$\text{mkt\_mean}_{t}^{(h)}=\frac{1}{h}\sum_{k=0}^{h-1}\tilde R_{t-k}.$$
				\end{itemize}
				The horizons $h_1\in\{5,10,21,63,126,252\}$ correspond to weekly,
				bi‑weekly, monthly, quarterly, semi‑annual and annual windows and $h\in\{3,5,10,21\}$
				
				These engineered columns give TabICL a rich, approximately i.i.d. tabular view
				of each forecasting episode while retaining the financial intuition behind each
				statistic. 
				\noindent
				Of course, these engineered columns are relatively \emph{toy} – we could readily
				enrich the representation with additional information such as trading volume,
				standard technical indicators (e.g., moving‑average convergence/divergence,
				relative strength index...), and other signals that are commonly
				used in production‑grade trading systems.  However, the purpose of the present
				study is not to devise a profitable trading strategy; rather, we aim to
				quantify how pre‑training on synthetic data influences the downstream
				performance of a tabular foundation model on financial data.  Consequently, we deliberately keep the
				feature set minimal and focus on the effect of the synthetic‑data augmentation
				itself.
				
				\subsubsection{Dimensionality Reduction}
				
				The training dataset consists of a single multivariate time series of length $N = 2263$
				and dimension $d = 433$. Rather than working directly with the high-dimensional return
				matrix $X \in \R^{N \times d}$, we project the data onto a lower-dimensional factor space
				$F \in \R^{N \times m}$ using principal component analysis (PCA), with $m \ll d$. In our
				experiments, we found $m = 16$.
				
				The extracted independent factors are subsequently grouped into $3$ clusters using $k$-means
				clustering, under the assumption that factors within the same cluster share the same distribution. The SBBTS model is then fitted independently to each cluster
				of factors. The remaining idiosyncratic components are treated separately. Since these residuals
				exhibit heavy-tailed behavior, they are modeled independently across dimensions using
				a Gaussian mixture with two components. 
				
				Synthetic samples of asset returns are recovered from the generated factor time series
				via the decomposition
				\[
				\hat{X}
				=
				\underbrace{\hat{F} P_{1:m}^\top}_{\text{factor component}}
				+
				\underbrace{\hat{R}}_{\text{residual component}},
				\]
				where $\hat{F} \in \R^{N \times m}$ denotes the synthetic factor matrix,
				$P_{1:m} \in \R^{d \times m}$ is the PCA projection matrix, and
				$\hat{R} \in \R^{N \times d}$ represents the synthetic residual time series. For further details on the dimensionality reduction procedure, we refer the reader to
				\cite{cetingoz2025syntheticdataportfoliosthrow}.

				In practice, we employ a sliding window approach with a stride of $1$ to decompose each cluster of factors in samples of length $253$. This yields a training set, denoted as $\mathcal{D}$, on which we fit the SBBTS model. Additionally, we generate synthetic samples of the S\&P 500, also of length $253$, and split each synthetic sample $\hat{X} = (\hat{X}_{t_1}, \ldots, \hat{X}_{t_{253}})$ into input and target components, defined as:
				\begin{equation}
					\text{Input} = (\hat{X}_{t_1}, \ldots, \hat{X}_{t_{252}}) \quad \text{and} \quad \text{Target} = \mathrm{sign}(\hat{X}_{t_{253}}).
				\end{equation}
				
    \subsubsection{Discussion on Sharpe ratios}
    \label{subsec:sup_res_trad}
    
    We investigate the statistical significance of the Sharpe ratios obtained in our experiments.
    More specifically, we compute \textbf{95\% bootstrap confidence intervals} for the estimated Sharpe ratios using the methodology proposed in \cite{Riondato2018SharpeRE}.
    We focus on the validation and test sets, each comprising $420$ i.i.d.\ observations, and compare the real-only training regime with the setting augmented using \textsc{SBBTS} synthetic data.
    The resulting confidence intervals are reported in Table~\ref{tab:ic_comparaison}.
    
    \begin{table}[H]
    \centering
    \begin{tabular}{ccc}
    \toprule
     & Real & \textsc{SBBTS} \\
    \midrule
    Validation & $[-1.73,\ 1.28]$ & $[-0.62,\ 2.30]$ \\
    Test       & $[0.28,\ 3.32]$  & $[0.58,\ 3.34]$ \\
    \bottomrule
    \end{tabular}
    \caption{95\% bootstrap confidence intervals for the Sharpe ratios obtained under real-only training and \textsc{SBBTS}-based data augmentation.}
    \label{tab:ic_comparaison}
    \end{table}
    Although the confidence intervals obtained under \textsc{SBBTS} augmentation consistently exhibit higher upper and lower bounds (suggesting that in the worst-case scenario, the model trained on SBBTS data performs less poorly), the overlap between intervals prevents us from concluding that the improvement in Sharpe ratio is \textit{statistically significant} at conventional confidence levels.
    This limitation is primarily due to the relatively small number of observations available in both the validation and test sets.
    In practice, establishing statistical significance for Sharpe ratios typically requires a much larger sample size. 
				
\small 

\bibliographystyle{plain}

\bibliography{biblioSBB}

\end{document}